\documentclass{article} 
\usepackage{iclr2027_conference,times}

\usepackage{amsmath,amsfonts,bm}

\def\eqref#1{equation~\ref{#1}}

\def\1{\bm{1}}

\DeclareMathAlphabet{\mathsfit}{\encodingdefault}{\sfdefault}{m}{sl}
\SetMathAlphabet{\mathsfit}{bold}{\encodingdefault}{\sfdefault}{bx}{n}

\usepackage[utf8]{inputenc}
\usepackage[T1]{fontenc}
\usepackage{hyperref}
\usepackage{url}
\usepackage{booktabs}
\usepackage{amsfonts}
\usepackage{amsmath}
\usepackage{mathtools}
\usepackage{amssymb}
\usepackage{amsthm}
\newtheorem{proposition}{Proposition}
\newtheorem{lemma}{Lemma}
\usepackage{nicefrac}
\usepackage{microtype}
\usepackage{xcolor}
\usepackage{graphicx}
\usepackage{float}
\usepackage[noabbrev]{cleveref}
\usepackage{changes}
\crefname{figure}{Figure}{Figures}
\crefname{section}{Section}{Sections}
\crefname{appendix}{Appendix}{Appendices}
\crefname{table}{Table}{Tables}
\crefname{equation}{Eq.}{Equations}

\title{Decision-Aware Training for Sample-Based\\Generative Models}

\author{Kornelius Raeth$^{1}$ \quad Nicole Ludwig$^{1,2}$ \\
$^{1}$Tübingen AI Center, University of Tübingen, Germany \\
$^{2}$University of Augsburg, Germany \\
\texttt{kornelius.raeth@uni-tuebingen.de}
}

\begin{document}

\maketitle
\lhead{Preprint} 

\begin{abstract}
Sample-based generative models are increasingly used for probabilistic forecasting in high-stakes decision settings, yet their training objectives are blind to the decision maker’s cost structure.
These models are commonly trained with strictly proper scoring rules, such as the energy score, which allocate their training signal in proportion to data density, 
with no awareness of where forecast errors are most costly for downstream decisions.
We therefore propose \emph{decision-aware training for sample-based generative models}, augmenting the energy score objective with a differentiable decision loss that directly penalises the cost incurred by acting on the model's forecast.
This combined loss is theoretically grounded, as the decision loss is itself a proper scoring rule.
We compute the decision loss via a differentiable optimisation layer. 
Its gradient concentrates in cost-sensitive regions of the output space, making the method's effects interpretable and predictable from the cost structure.
We validate the method on one synthetic and two real-world tasks. In the synthetic task, the method corrects the mode weights of a learned bimodal distribution; 
in a wind power dispatch task, it concentrates improvements in the rare but costly tail region, and in a frost protection task, it improves how well the decision costs are anticipated. 
Our method yields generative models that retain full probabilistic forecasts while being better aligned with the decision maker's specific cost structure.
\end{abstract}

\section{Introduction}
\label{sec:intro}

Probabilistic forecasting models are increasingly deployed in high-stakes decision-making, for example, a wind farm operator committing to a power contract based on a wind forecast, risking costly shortfall penalties if they over-commit, or a farmer deciding whether to apply frost protection based on a temperature forecast, balancing asset loss against the cost of intervention.
In the weather domain, state-of-the-art forecasting systems increasingly rely on sample-based generative models~\citep{lang2024b, bonev2025fourcastnet}. These models are trained with strictly proper scoring rules such as the energy score~\citep{gneiting2007strictly}, which are uniquely minimised by the true conditional distribution.
In practice, however, limited model capacity and empirical risk minimisation yield only local optimality~\citep{blasiok2023}, 
inducing tradeoffs between model error in different regions of the input/output space~\citep{donti2017task}: proper scoring rules distribute the training gradient in proportion to data density, with no awareness of the decision maker's cost structure.
The model's limited capacity is allocated globally, leaving decision-critical regions of the output space potentially underserved.

Given a forecast, a decision maker with cost function $c(a, y)$, of action $a$ and outcome $y$, selects the action that minimises expected cost under the forecast distribution;
a point forecast is insufficient to evaluate this expectation.
A good forecast should yield low expected costs when acted upon and correctly anticipate those costs, the latter property being known as \emph{decision calibration}~\citep{zhao2021}.
Crucially, the observed cost of the optimal action is itself a proper scoring rule~\citep{hartline2025smooth, kleinberg2023u}, placing it in the same family as the energy score which licenses their combination as a theoretically well-founded training objective for decision-aware training.

We augment the energy score with the optimal action cost as a differentiable decision loss.
The decision loss alone is insufficient as a training objective: it only constrains the model in cost-sensitive regions, leaving the rest of the distribution unanchored and prone to degeneration. The energy score acts as that anchor, preventing the model from collapsing outside cost-sensitive regions.
During training, samples drawn from the model are used to compute the optimal action $a^*$ via a differentiable optimization layer; the decision loss then evaluates the cost incurred by $a^*$ against the true outcome $y$, and gradients flow back through $a^*$ to the model parameters via implicit differentiation.
Our method is theoretically grounded and leads to better downstream decisions while retaining full probabilistic forecasts, as validated on synthetic and real-world forecasting tasks.

\textbf{Contributions.} 
\begin{itemize}
    \item An end-to-end \emph{decision-aware} training framework for implicit generative models and distributional diffusion models,
    bridging decision-focused learning and probabilistic forecasting by retaining distributional accuracy via an energy score anchor.
    \item A theoretically principled combination of decision loss and energy score, exploiting the equivalence between optimal action costs and proper scoring rules.
    \item A gradient analysis showing which regions benefit from the decision loss and why, based on the cost function structure.
    \item Empirical validation of our method on synthetic and real-world forecasting tasks, considering both decision loss and decision calibration.
\end{itemize}
In the following, we start with summarising the necessary background in \cref{sec:forecasting,sec:decision}  before introducing our method in \cref{sec:method} and validating it in \cref{sec:experiments}.

\section{Probabilistic Regression and Scoring Rules}
\label{sec:forecasting}

We consider a probabilistic regression setting where input features $X \in \mathcal{X}$ and target $Y \in \mathcal{Y}$ are random variables distributed according to $P_{X,Y}$. 
Given a realisation $x$ of $X$, the goal is to learn a model $h_\theta$ that produces a predictive distribution $\hat{F}_x = h_\theta(x)$ over $Y$ which ideally recovers the true conditional $Y \mid X = x$.
We drop the subscript $x$ throughout for brevity.

For this work, we consider two classes of sample-based generative models: implicit generative models, which map noise directly to samples, and distributional diffusion models, which generate samples via an iterative denoising process. Both represent $\hat{F}_x$ implicitly via samples $\{\hat{y}_m\}_{m=1}^M$.
\emph{Implicit generative models}~\citep{mohamed2016learning} map a conditioning input $x$ and a noise vector $\varepsilon_m \sim \mathcal{N}(0, I)$ directly to a sample $\hat{y}_m \sim \hat{F}_x$ via $\hat{y}_m = f_\theta(x, \varepsilon_m)$; drawing $M$ independent noise vectors in parallel yields $M$ samples in a single forward pass.
\emph{Diffusion models}~\citep{ho2020denoising, song2020score} gradually add noise to the clean target $y_0$ and learn a reverse denoising process conditioned on $x$.\footnote{We use $y_0, y_t$ for the clean and noisy target following our regression notation; the standard diffusion literature uses $x_0, x_t$.}
At each reverse step, a denoiser estimates the clean target $y_0$ from a noisy observation $y_t$ and the conditioning input $x$; in the standard formulation this prediction is a point estimate.
Since $y_0$ is uncertain given $(y_t, x)$, we consider the special class of \emph{distributional diffusion models}~\citep{debortoli2025distributional, kneissl2025improved}, whose denoiser models a full distribution over $y_0$.
At inference, samples from $\hat{F}$ are obtained by iterating the full reverse chain.

A scoring rule $S(\hat{F}, y)$ assigns a scalar score to a predictive distribution $\hat{F}$ given the realised outcome $y$.
It is \emph{proper} if the true conditional $F$ minimises the expected score, $\mathbb{E}_{Y \sim F}[S(F, Y)] \leq \mathbb{E}_{Y \sim F}[S(\hat{F}, Y)]$  for all $\hat{F}$, and \emph{strictly proper} if $F$ is the unique minimiser.
The \emph{energy score} (ES)~\citep{gneiting2007strictly} is a popular example of a strictly proper scoring rule that admits a sample-based estimator (see \cref{app:energy_score}). For univariate targets the energy score reduces to the \emph{continuous ranked probability score} (CRPS), widely used in meteorological forecast verification~\citep{gneiting2007} and as a training objective for state-of-the-art ML weather prediction models~\citep{lang2024b, price2023}.
Both model classes introduced above are commonly trained by minimising strictly proper scoring rules, and in particular the energy score.

\section{Decision-Making Under Uncertainty}
\label{sec:decision}

Given a probabilistic forecast $\hat{F}$ over $\mathcal{Y}$, a decision maker selects an action $a \in \mathcal{A}$ to minimise expected cost under the forecast.
This is modelled via a cost function $c : \mathcal{A} \times \mathcal{Y} \to \mathbb{R}$ that encodes the preferences of a particular decision maker; the Bayes-optimal action is
\begin{align}
    a^*(\hat{F}) &\coloneqq \operatorname*{argmin}_{a \in \mathcal{A}} \, \mathbb{E}_{Y \sim \hat{F}}[c(a, Y)] \label{eq:optimal_action}\\
    &\approx \operatorname*{argmin}_{a \in \mathcal{A}} \frac{1}{M}\sum_{m=1}^M c(a, \hat{y}_m)= a^*(\hat{F}_M) \label{eq:optimal_action_mc}.
\end{align}
The expectation ensures that rare but sufficiently costly events still influence the optimal action.
For sample-based models, the expectation is replaced by its Monte Carlo estimate over samples $\hat{y}_m \stackrel{\text{iid}}{\sim} \hat{F}$, $m = 1,...,M$, and $\hat{F}_M$ denotes the empirical distribution\footnote{$\hat{F}_M \coloneqq \frac{1}{M}\sum_{m=1}^M \delta_{\hat{y}_m}$.} over $\{\hat{y}_m\}_{m=1}^M$.

The realised cost when acting on forecast $\hat{F}$ and observing the true outcome $y$ is $c(a^*(\hat{F}), y)$.
A forecast is \emph{decision-calibrated}~\citep{zhao2021} if its predicted expected cost matches the realised cost in expectation over inputs:
\begin{equation}
    \mathbb{E}_{X}\,\mathbb{E}_{Y \sim \hat{F}_X}\bigl[c(a^*(\hat{F}_X),\, Y)\bigr] = \mathbb{E}_{(X,Y)}\bigl[c(a^*(\hat{F}_X),\, Y)\bigr].
    \label{eq:dec_calibration}
\end{equation}
Intuitively, that means a decision-calibrated model correctly anticipates the cost of its own optimal action. We measure decision miscalibration as the mean absolute difference between predicted cost and realised cost across test points $\frac{1}{N} \sum_{i=1}^N \left| \mathbb{E}_{Y \sim \hat{F}_{x_i}}\bigl[c(a^*(\hat{F}_{x_i}), Y)\bigr] - c(a^*(\hat{F}_{x_i}), y_i) \right|$.

In principle, minimising a strictly proper scoring rule leads to optimal decisions: $F$ is the unique minimizer of the expected score~\citep{gneiting2007strictly}, and the true distribution yields optimal decisions by construction.\footnote{In fact, distribution calibration~\citep{song2019} would suffice for optimal decisions, but even this weaker condition is not attainable nor verifiable in practice, motivating a focus on decision calibration directly~\citep{derr2025, sahoo2021}.}
In practice, however, limited data, model capacity, and stochastic optimisation yield only approximate local minima~\citep{blasiok2023}, with no awareness of where $c$ is sensitive~\citep{donti2017task}, so a model minimising the energy score may be decision-miscalibrated and incur high decision cost.

\section{Decision-Aware Training}
\label{sec:method}

\begin{figure}[t]
    \centering
    \includegraphics[width=0.85\linewidth]{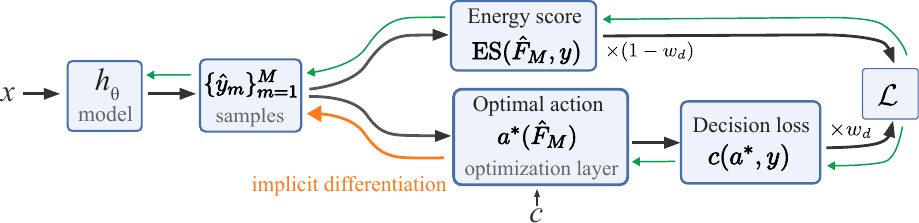}
    \caption{
        \textbf{Method overview.}
        In the forward pass the model $h_\theta$ produces samples $\{\hat{y}_m\}$ given input $x$.
        These feed two branches: the \emph{Energy Score} (CRPS) computes the standard scoring rule loss given the observed target $y$;
        the \emph{Optimization layer} solves \cref{eq:optimal_action} to obtain $a^*$, and the \emph{decision loss} evaluates $c(a^*, y)$.
        Gradients flow back from energy score via standard autodiff (green arrows) and from decision loss through the optimization layer via implicit differentiation (orange arrow).
    }
    \label{fig:method_overview}
\end{figure}

We address the gap between scoring-rule training and downstream decision quality by augmenting the scoring rule objective with a differentiable decision loss that directly penalises suboptimal decisions (\cref{fig:method_overview}).
We first establish the theoretical grounding and introduce the combined objective (\cref{sec:dec_loss}), 
then analyse the gradient chain through the optimal action (\cref{sec:gradient}), 
and describe the model architectures and training procedure (\cref{sec:architectures_training}).

\subsection{Loss formulation and theoretical grounding}
\label{sec:dec_loss}

The decision problem of \cref{sec:decision} suggests a natural training signal: penalise the model's optimal action by the cost it incurs at the realised outcome.
We define the \emph{decision loss}
\begin{equation}
    S_c(\hat{F}, y) = c\!\left(a^*(\hat{F}),\, y\right),
    \label{eq:dec_loss}
\end{equation}
where $a^*(\hat{F})$ is the Bayes-optimal action in~\cref{eq:optimal_action}.
The key observation is that $S_c$ is itself a proper scoring rule.
\begin{proposition}[\citealt{gneiting2007strictly}; \citealt{hartline2025smooth}]
\label{prop:psr}
The decision loss $S_c$ is a proper scoring rule: $\mathbb{E}_{Y \sim F}[S_c(F, Y)] \leq \mathbb{E}_{Y \sim F}[S_c(\hat{F}, Y)]$ for all forecasts $\hat{F}$, which follows from the definition of $a^*(F)$ as the Bayes action under $F$.
However, it is not strictly proper: there can be $\hat{F} \neq F$ with $a^*(\hat{F}) = a^*(F)$ that achieve the same minimum.
\end{proposition}
This means $S_c$ is part of the same loss family as the energy score.\footnote{In fact, \citet{hartline2025smooth} build on \citet{kleinberg2023u} to establish the stronger result that the space of decision losses induced by bounded cost functions is \emph{equivalent} to the space of proper scoring rules: every proper scoring rule arises as the optimal action cost of some cost function. Our argument only requires the direction stated in \Cref{prop:psr} ($S_c$ is proper).}
However, since the minimum is not unique, optimising $S_c$ alone admits degenerate solutions.
The energy score provides the missing uniqueness: as a strictly proper scoring rule, it admits $F$ as its unique minimiser, and a positive linear combination of a strictly proper and a proper scoring rule remains strictly proper (\Cref{lemma:strict_properness}, \cref{app:proof_strict_properness}).
This result directly applies to the training of sample-based generative models, offering a natural extension to the existing scoring rule training objective.

\textbf{Combined loss.}
Hence, we define our loss as convex combination of energy score and decision loss
\begin{equation}
    \mathcal{L} = (1 - w_d)\,\mathrm{ES} + w_d\,S_c\, , \qquad w_d \in [0, 1),
    \label{eq:combined}
\end{equation}
with a unique minimum at $\hat{F} = F$. The decision weight $w_d$ trades off distributional accuracy against decision awareness.\footnote{We discuss heuristics for choosing $w_d$ in practice in Appendix~\ref{app:wd_guidance}.}
The energy score serves as an anchor, eliminating degenerate solutions admitted by $S_c$ alone, while $S_c$ concentrates the gradient on the cost-relevant region. This combination yields a theoretically grounded objective that is both decision-aware and distribution preserving.

\subsection{Optimizing the decision loss}
\label{sec:gradient}

Training with the decision loss $S_c$ requires computing $a^*$ from samples via \cref{eq:optimal_action_mc}; the backward pass must then differentiate $a^*$ with respect to those samples.
For general cost functions, $a^*$ has no closed-form expression in the samples $\{\hat{y}_m\}$, so we cannot differentiate it with respect to $\{\hat{y}_m\}$ directly; instead we differentiate the optimality conditions that define $a^*$ treating it as an implicit function of $\{\hat{y}_m\}$ (implicit differentiation).
We compute $a^*$ within a \emph{differentiable optimization layer}~\citep{amos2017optnet, donti2017task, blondel2022efficient} over the feasible set $\mathcal{A} = [a_\mathrm{min}, a_\mathrm{max}]$, treating model samples as fixed.
This requires $c$ to be differentiable in $a$ and the expected cost in~\cref{eq:optimal_action} to have positive curvature at $a^*$, allowing implicit differentiation through $a^*$.

The decision loss $S_c = c(a^*, y_\mathrm{obs})$ depends on model parameters only through the samples $\{\hat{y}_m\}$ and hence through $a^*$.
The gradient with respect to each sample is
\begin{equation}
    \frac{\partial S_c}{\partial \hat{y}_m} = \left.\frac{\partial c}{\partial a}\right|_{a^*} \cdot \frac{\partial a^*}{\partial \hat{y}_m},
    \label{eq:grad_chain}
\end{equation}
where $\partial a^* / \partial \hat{y}_m$ is obtained via implicit differentiation, whereas the remaining upstream gradient with respect to model parameters is computed via standard automatic differentiation.

\textbf{Per-sample gradient.}
When $a^* \in (a_\mathrm{min}, a_\mathrm{max})$, implicit differentiation gives
\begin{equation}
    \frac{\partial a^*}{\partial \hat{y}_m} = -\left[\frac{1}{M}\sum_{j=1}^M \frac{\partial^2 c}{\partial a^2}(a^*, \hat{y}_j)\right]^{-1} \frac{1}{M}\, \frac{\partial^2 c}{\partial a\, \partial \hat{y}_m}(a^*, \hat{y}_m).
    \label{eq:ift_member}
\end{equation}
The gradient signal is concentrated on samples $\hat{y}_m$ where the cost has non-zero cross-curvature, $\partial^2 c / \partial a \partial \hat{y}_m \neq 0$, i.e., where small perturbations of $\hat{y}_m$ change how sensitive the cost is to $a$; samples in flat or saturated cost regions contribute almost nothing (full derivation in \cref{app:ift_derivation}).
In sample-based models, the noise input $\varepsilon$ enables flexible distributional shapes, and gradients can update specific parts of the predictive distribution. 
However, this flexibility can also lead to distributional degeneration at large $w_d$ in practice.

\textbf{Constraint boundary blocking.}
~\Cref{eq:ift_member} applies when $a^*$ is interior, where the first-order condition $\nabla_a \frac{1}{M} \sum_{m=1}^M c(a^*, \hat{y}_m) = 0$ holds.
When $a^*$ reaches a constraint boundary ($a^* \in \{a_\mathrm{min}, a_\mathrm{max}\}$), \cref{eq:ift_member} does not apply. In this case the gradient is simply zero: 
geometrically, small perturbations of the samples do not move $a^*$ off the boundary, so $\partial a^*/\partial \hat{y}_m = 0$ (formal derivation via KKT conditions in \cref{app:ift_derivation}), 
blocking the decision loss gradient chain in~\cref{eq:grad_chain} entirely.\footnote{The measure-zero case where the unconstrained minimiser coincides exactly with the boundary is excluded.}
The energy score is then the only active training signal.

\subsection{Model architectures and training procedure}
\label{sec:architectures_training}

For this work, we consider two sample-based generative architectures representative of those used in operational forecasting systems \citep[eg.,][]{lang2024, couairon2024}.
We extend both architectures with a differentiable optimization layer that computes $a^*$ from the model samples.

\textbf{Implicit generative model.}
An MLP maps input $x$ and $M$ independent noise draws to $M$ samples in a single forward pass (details in \cref{app:implementation}).

\textbf{Distributional diffusion model.}
An MLP denoiser produces $M$ samples of $y_0$ at each denoising step. The decision loss acts at the denoiser level, while the predictive distribution $\hat{F}_x$ at inference is shaped by the full reverse chain, an additional sampling layer absent in the implicit generative model.

\textbf{Training procedure.}
At each training step, the model is trained by optimising the combined loss (\Cref{eq:combined}).
In practice, both loss terms are evaluated from a finite set of samples $\{\hat{y}_m\}_{m=1}^M$ produced by the model for each input $x$ in the batch (\cref{fig:method_overview}).
First, the energy score is evaluated using the unbiased fair estimator (see Appendix \cref{app:energy_score}), with gradients computed via standard autodiff. Second, $a^*$ is computed by minimising the expected cost over the same samples (\Cref{eq:optimal_action_mc}) via the optimisation layer forward pass, treating the model samples as fixed.
The decision loss $S_c$ is then evaluated and its gradient flows back through $a^*$ via implicit differentiation (\cref{sec:gradient}).
Computing $a^*$ from $M$ samples is the only source of a $O(1/M)$ finite-sample bias in the objective that vanishes as $M \to \infty$x; we show that the objective nevertheless remains approximately strictly proper and that the bias is small in practice (\cref{app:finite_m}).
Both losses are normalised to unit scale before training, so that $w_d$ directly controls the relative weighting between the two terms independently of the cost function's magnitude; further implementation details are given in \cref{app:implementation}.
\section{Experiments}
\label{sec:experiments}
In the following we evaluate our method on one synthetic and two real-world decision tasks.
\subsection{Synthetic Decision Task}
\label{sec:synthetic}
\begin{figure}[t]
    \centering
    \includegraphics[width=\linewidth]{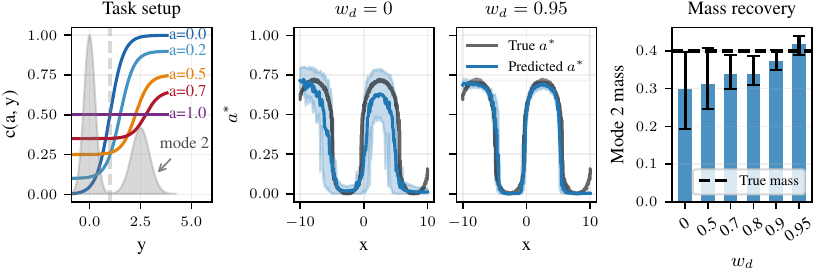 }
    \caption{
        \textbf{Synthetic decision task.}
        \emph{Left:} Cost function $c(a, y)$ for five protection levels $a$, overlaid with the marginal $p(y)$ (grey). The threshold at $y=1.0$ separates the two modes; mode positions are fixed across $x$, only the mixture weight varies.
        \emph{Center (two panels):} Predicted $a^*(x)$ for $w_d=0$ and $w_d=0.95$. Pure energy score ($w_d=0$, left) fails to track the ground truth at the transitions; decision-aware training  (ours) recovers the correct shape ($w_d=0.95$, right).
        \emph{Right:} Marginal mode 2 mass vs.\ $w_d$; dashed line marks the true mass ($0.4$). Average mass approaches the true value for larger decision weights $w_d$.
        Statistics show mean and 10th--90th percentiles across seeds.
    }
    \label{fig:synthetic}
\end{figure}

\textbf{Setup.}
To study the mechanism underlying decision-aware training in a controlled setting where the ground-truth conditional distribution is known, we construct a synthetic dataset with the following structure. 
The target $Y \mid X$ follows a bimodal Gaussian mixture distribution with fixed mode positions ($\mu_1 = 0$, $\mu_2 = 2.5$), fixed standard deviations ($\sigma_1 = 0.2$, $\sigma_2 = 0.4$), 
but a sinusoidally varying mixture weight $w(x)$, with mode 2 (the right mode at $y = 2.5$ in \Cref{fig:synthetic} left) having lower marginal weight $0.4$.
The decision maker must choose a protection level $a \in [0, 1]$ against a hazard concentrated above a threshold at $y = 1.0$ (between the two modes): higher $a$ reduces exposure to the hazard but incurs a growing protection cost.
The threshold is set so that the mixture weight $w(x)$ is the key determinant of the optimal action and mode-weight tracking, thus the key diagnostic.
We train a distributional diffusion model on this task across multiple data and training seeds; details are in \cref{app:synthetic}.
We also use this synthetic setup to investigate the effect of training data size on the benefit of decision-aware training in Appendix~\Cref{app:scale}.

\textbf{Results.}
With this setup in place, we first ask whether a standard proper scoring rule suffices. Under pure energy score training ($w_d=0)$, the diffusion model roughly captures the extreme values of $a^*(x)$ but fails to reproduce the sinusoidal transitions of $a^*$ between extremes, reflecting inaccurate mode weight tracking across $x$ (\cref{fig:synthetic}).
Inspecting the mode weights of the marginal predictive distributions shows that the costly mode (mode 2 in \Cref{fig:synthetic}) is systematically underweighted when trained with energy score alone ($w_d=0$).

Adding the decision loss directly addresses this: $a^*(x)$ tracks the ground truth more closely as $w_d$ grows (\cref{fig:synthetic}). 
The marginal mode weights are progressively corrected as $w_d$ increases, with probability mass shifting toward mode 2, improving decision skill. 
At the extreme $w_d=1$ (no energy score anchor), $a^*$ tracking survives (\Cref{fig:synthetic_action_curves}) but the predictive distribution degenerates, confirming empirically that the energy score is necessary as a distributional anchor (\Cref{fig:synthetic_samples_scatter}).
Results for all data seeds and aggregate metrics are provided in \cref{app:synthetic}.

\subsection{Wind Power Dispatch}
\label{sec:windpower}
\textbf{Setup.}
Wind power output is inherently uncertain, thus wind power operators must commit to a power production level (dispatch) hours before realised output is known; over-commitment incurs shortfall penalties while under-commitment forgoes revenue~\citep{bruninx2025, botterud2011, bourry2008}.
The power curve $P$ maps wind speeds to power output and introduces a non-trivial complication: at high wind speeds, turbines are shut down for safety (cut-off event), so production drops from peak output to zero; this cut-off regime is the most consequential for the dispatch decision since a missed cut-off event incurs the largest shortfall penalty.
Accurate probabilistic modelling of this regime is therefore crucial for optimal decision-making.
We study wind power dispatch at an offshore wind farm (North Sea, Norway), where the decision maker must commit to a normalized power dispatch level $a \in [0, 1]$ based on a forecast.
We model this task via a cost function with quadratic shortfall penalty~\citep{bruninx2025}
\begin{equation}
    c(a, y) = -a + \lambda \, \mathrm{relu}(a - P(y))^2,
\end{equation}
where $-a$ reflects the revenue from committing to dispatch level $a$ (higher commitment earns more), $P(\cdot)$ is a differentiable power curve  (\cref{fig:windpower_setup}), and $\lambda$ controls the shortfall penalty.
The modelled power curve has three regions: a \emph{ramp} region where output rises with wind speed (3--15\,m/s at turbine hub height), a \emph{rated} region of flat maximum output (15--22\,m/s), and a \emph{cut-off} region where output drops sharply to zero above $\approx$22\,m/s ($\approx$5\% of observations).
We forecast wind speed (at 10\,m)  6\,h ahead from ERA5 reanalysis data~\citep{hersbach2020}, using lagged wind speed observations as inputs.
Forecast samples are extrapolated to turbine hub height speed, passed through $P(\cdot)$, and $a^*$ is solved via the optimisation layer; the observed cost is then computed using the true observed power (see \cref{app:windpower} for details).
We train a distributional diffusion model and an implicit generative model across $w_d \in \{0.0, 0.1, 0.3, 0.5, 0.7, 0.9, 0.95\}$ with three seeds (training details in \cref{app:windpower}).
We sweep $\lambda \in \{3, 5, 10\}$ and report $\lambda = 5$ as the primary setting.

\textbf{Results.}
\begin{figure}[t]
    \centering
    \includegraphics[width=\linewidth]{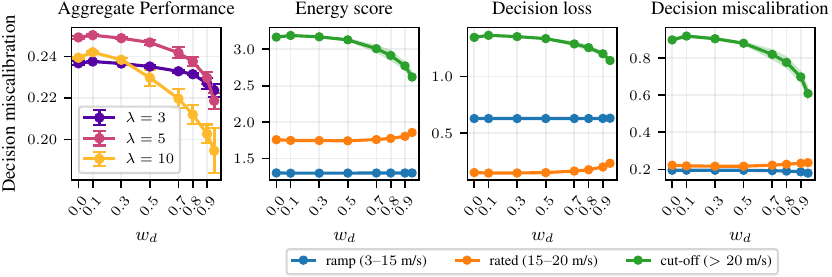}
    \caption{
        \textbf{Wind power dispatch results} (diffusion model).
        \emph{Left:} aggregate decision miscalibration vs $w_d$ for all $\lambda$ (mean $\pm$ 1\,std across seeds). Increasing $w_d$ improves decision miscalibration. This effect scales with shortfall penalty $\lambda$.
        \emph{Right (three panels):} conditional metrics by power-curve region of the observed wind speed vs.\ $w_d$ (mean $\pm$ 1\,std across seeds) for $\lambda = 5$. 
        Improvements concentrate in the cut-off region; the rated region degrades for large $w_d$; the ramp region is mostly unaffected.
    }
    \label{fig:windpower_main}
\end{figure}
At the aggregate level, decision-aware training improves decision miscalibration monotonically for $w_d > 0.5$ in the diffusion model (\cref{fig:windpower_main}, left), with larger $\lambda$ yielding greater improvements. The energy score also improves with $w_d$ for $\lambda \leq  5$, but degrades at large $w_d$ for $\lambda=10$, while aggregate decision loss remains relatively flat throughout (\cref{fig:windpower_aggregate_diffusion}).
The implicit generative model shows behaves similarly (\cref{fig:windpower_aggregate_implicit}), except for the energy score, which degrades monotonically for all $\lambda$, indicating a clearer tradeoff between distributional accuracy and decision quality.

These aggregate metrics, however, obscure where the decision loss acts. Disaggregating performance by power-curve region of the observed wind speed (\cref{fig:windpower_main}) shows that the effect concentrates in the cut-off region\footnote{We use 20\,m/s as the analysis threshold rather than the physical cut-off speed of 22\,m/s to include events near the cut-off threshold where power already falls sharply.}, where forecast errors are most costly. At $w_d = 0.9$, the energy score improves by 12\%, decision loss by 11\%, and decision calibration by 22\% on average compared to $w_d=0$. 
This gain comes at the cost of a slight monotonic degradation in the rated region, with no effect in the ramp region.
For the implicit generative model, the regional picture is less clean. The baseline energy score model is already well calibrated in the cut-off region, leaving little room for the decision loss to exploit; only $\lambda=10$ yields a consistent regional improvement (see \cref{fig:windpower_regional_implicit}).

Two further observations are notable for the correction mechanism. First, the cut-off improvements in the diffusion model scale monotonically with $\lambda$ (\Cref{fig:windpower_cutoff_diffusion}), as a higher penalty produces stronger corrections in the tail. However, at very high $w_d$, the energy score degrades again, suggesting the decision loss overshoots the optimal distributional correction. Second, the correction is directly visible as a median shift $\Delta Q_{0.5}(v)$ in the predictive distributions of both models (\Cref{fig:windpower_rolling_q_implicit,fig:windpower_rolling_q_diffusion}). Corrections are negative in the ramp region, corresponding to more conservative dispatch, and positive beyond the cut-off threshold, placing more mass in the tail where power drops to zero.

\textbf{Mechanism.}
The regional pattern of improvements and degradations follows from how the decision-loss gradient propagates through the power curve.
In the \emph{ramp} region ($\partial P/\partial y > 0$, $a^*$ interior), gradients propagate and produce a shift to more conservative values in the predictive distribution, but aggregate metrics remain unaffected.
In the \emph{rated} region, $\partial P/\partial y \approx 0$ suppresses the decision loss gradient, leaving the energy score at reduced weight $(1-w_d)$, which explains the degradation.
In the \emph{cut-off} region, $\partial P/\partial y \ll 0$ and $a^*$ is interior; the gradient is strong, providing explicit supervision on the rare, high-cost tail events that energy score under-weights due to their low frequency.

\subsection{Frost Protection}
\label{sec:frost}
\textbf{Setup.}
Frost is among the most economically damaging meteorological hazards~\citep{juurakko2021, white1975}. Decision makers must commit to protection measures before observing temperatures, but protection comes at a cost and should be applied only when required.
We study this frost protection task in Berlin using ERA5 reanalysis data~\citep{hersbach2020}, forecasting temperature (at 2\,m, T2m) 24\,h ahead from lagged ERA5 temperature observations over a spatial patch, restricted to winter months.
The decision maker chooses a frost protection level $a \in [0, 1]$ before observing the realised temperature $y$. The cost function has the same sigmoid threshold structure as in \cref{sec:synthetic}, but mirrored: the hazard lies below the threshold.
It penalises both unprotected frost (false negatives, FN) and unnecessary protection (false positives, FP), with $\alpha$ setting the FP/FN cost ratio; lower $\alpha$ penalises unprotected frost more strongly. 
Unlike the rare cut-off event in \cref{sec:windpower}, frost occurs in about 25\% of observations, making this a common event, where the energy score is less likely to under-weight the relevant area. 
We train an implicit generative model and a distributional diffusion model across $w_d \in \{0.0, 0.1, \ldots, 0.9\}$ and $\alpha \in \{0.2, 0.3, 0.5\}$ with three seeds (see details in \cref{app:frost_setup}); diffusion model results are in \cref{app:frost_diffusion}.

\begin{figure}[t]
    \centering
    \includegraphics[width=\linewidth]{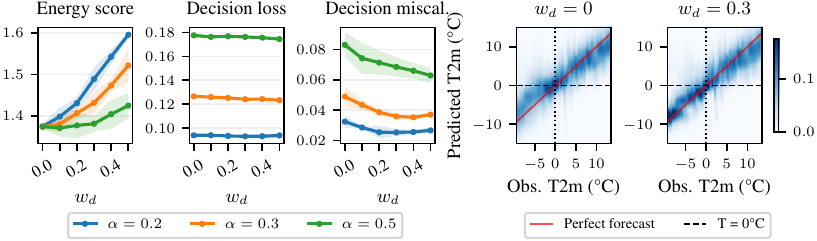}
    \caption{
        \textbf{Frost protection results} (implicit generative model).
        \emph{Left (three panels):} Aggregate energy score, decision loss, and decision miscalibration vs.\ $w_d$ for $\alpha \in \{0.2, 0.3, 0.5\}$ (mean $\pm$ std over seeds). Aggregate decision loss remains mostly flat with small improvement for $\alpha \geq 0.3$; energy score degrades with $w_d$; decision miscalibration improves most visibly for $\alpha=0.5$.
        \emph{Right (two panels):} Conditional predictive density for $w_d=0$ and $w_d=0.3$ ($\alpha=0.3$) conditioned on observed temperature. A warm bias present at $w_d=0$ ($T < 0$\,\textdegree C) is corrected at $w_d=0.3$, with the predictive distribution concentrating closer to the diagonal.
    }
    \label{fig:frost_main}
\end{figure}

\textbf{Results.}
At the aggregate level, the energy score degrades with $w_d$, more steeply for lower $\alpha$, while aggregate decision loss is mostly flat (\cref{fig:frost_main}). For $w_d > 0.5$, some forecasts exhibit spread collapse around the freezing point, which we attribute to excessive decision loss correction. 
Decision miscalibration nevertheless improves for all $\alpha$, most visibly for $\alpha=0.5$ with an improvement of approximately 15\% at $w_d = 0.3$--$0.5$. Consistent with this, forecasts of sub-zero temperatures align more closely with observations when trained with the decision loss (\cref{fig:frost_main}, right).

The combination of improved decision calibration and nearly flat aggregate cost is explained by the higher base rate of frost. Because frost events are common, the energy score already allocates substantial gradient signal to the threshold region and likely yields a near-optimal $a^*$ without the decision loss. The decision loss still sharpens cost estimation, as the calibration gains show, but its impact on aggregate cost is marginal. A conditional breakdown (\Cref{fig:frost_conditional}) shows that the no-frost regime drives the improvement in decision calibration, while the frost regime degrades slightly.

\textbf{Mechanism.}
Due to the sharp sigmoid cost function, the decision loss gradient is concentrated near the 0\,\textdegree C threshold. Therefore, quantiles closer to the threshold (cost-sensitive region) shift more than those further away, as shown by the quantile shift plot in \Cref{fig:frost_rolling_dq}, aligning with the gradient analysis in \cref{sec:gradient}.
Unlike in the wind power cut-off region, many forecasts here straddle the threshold, making the decision loss gradient signal dense and its effects visible at $w_d = 0.1$--$0.2$, whereas wind power cut-off improvements require $w_d \approx 0.5$--$0.7$ because cut-off events are rare.

\section{Related Work}
\label{sec:related}

The relevant prior work spans statistics, decision theory, and machine learning. These strands have not previously been connected; we combine them to enable decision-aware training.

\textbf{Decision-focused learning.}
Decision-focused learning integrates the downstream decision task into the training objective, most commonly in predict-then-optimise pipelines where a point prediction is followed by a combinatorial or linear programme.
~\citet{elmachtoub2022smart} and ~\citet{wilder2019melding} replace the prediction loss with a surrogate of the task loss. \citet{donti2017task} embeds a differentiable optimisation layer and trains end-to-end. All target point predictions rather than distributions, and evaluate only task loss; we additionally measure decision calibration~\citep{zhao2021}, which assesses whether the model correctly prices the cost of its own decisions.
Our setting differs fundamentally: we work with probabilistic generative models that must produce forecast distributions, and the decision maker evaluates expected cost under the full distribution.
To our knowledge, no existing decision-focused learning method operates in the sample-based setting.

\textbf{Weighted proper scoring rules.}
Weighted proper scoring rules~\citep{gneiting2011comparing, lerch2017forecaster} emphasize specific regions of the output space via a fixed, pre-specified weight function $w(y)$ that depends only on the observation $y$.
In a decision problem, however, the cost function $c(a,y)$ depends jointly on the action and the observation. As a result, the cost-relevant region depends on $a^*(\hat{F}_x)$, and thus on the forecast itself which varies across inputs: low-wind and high-wind forecasts require precision in entirely different parts of the output space. No static reweighting can represent this input-dependent shift.  
Our method requires no pre-specified weight: the decision loss gradient concentrates automatically on the cost-relevant region, which changes for different $a^*$ and therefore adapts to the predicted distribution. 
Static reweighting approaches therefore improve calibration in a fixed region, while our method identifies which region matters for the decision at hand, and calibrates there. Thus, they address a different problem and are not comparable as baselines.

\textbf{Differentiable optimisation.}
Differentiable optimisation layers~\citep{amos2017optnet, donti2017task, blondel2022efficient} enable gradients to flow through the solution of an optimisation problem embedded in a neural network layer.
We use this machinery to differentiate through the optimal action $a^*$, connecting the decision loss to model parameters via implicit differentiation.

\section{Conclusion}
\label{sec:conclusion}

Decision-aware training for generative models is theoretically grounded and empirically effective, but its impact is selective. Whether a region of the output space receives a corrective signal depends on the optimal action: interior optima steer approximation errors toward less costly regions; boundary optima leave the original scoring-rule objective in full control. We demonstrated the mechanism on synthetic data, where decision-aware training corrects mode weight misattribution and improves downstream decisions; on wind power dispatch, where it concentrates improvements in the rare but costly cut-off region; and on frost protection, where it redistributes probability mass around the freezing threshold to improve decision calibration. The benefit is most pronounced when data or capacity limits leave energy score with suboptimal approximation errors in cost-relevant regions, a regime characteristic of operational forecasting (\Cref{app:scale}).

A practical limitation is that the optimal decision weight $w_d$ is task-dependent and requires tuning: at moderate $w_d$, a sweet spot typically exists where decision loss improves with little degradation in distributional accuracy, but the right balance depends on the cost function and data regime.
The method also requires a dedicated model per decision-maker, since the cost function is embedded into the training objective.
The optimisation layer adds computational overhead at each training step; further details are provided in \cref{app:implementation}.
For distributional diffusion models, our gradient analysis (\cref{sec:gradient}) characterises the decision loss training signal per denoising step, but does not account for inference-time reverse chain dynamics; a score-function formulation for the distributional denoiser could clarify when and why the effect carries through, and remains an open theoretical question.

Extending decision-aware training to large-scale probabilistic weather prediction models~\citep{lang2024b, bonev2025fourcastnet} is an interesting direction for future work. These systems share the same sample-based scoring rule training setup and are deployed in operational contexts where the gap between forecast skill and decision skill is operationally consequential.

\section*{Reproducibility Statement}
All details required to reproduce our results, including code availability, data sources, model architectures, hyperparameters, and training procedures, are provided in \cref{app:implementation}.

\section*{Ethics statement}
Decision-aware training produces forecasts that are more actionable for specific downstream tasks, with direct applications in safety-critical domains such as weather hazard prediction, where better forecasts can directly reduce societal harm. A potential risk is that if the cost function encodes biased priorities, the model will optimise for them, potentially reinforcing skewed downstream decisions.

\section*{AI use statement}
We used generative AI tools (Claude, Anthropic, including the Claude Code coding assistant) in the following ways.
\emph{Theory:} for the analysis in \cref{app:finite_m}, to identify key ingredients of the argument and to draft the proof; for the gradient derivation in \cref{sec:gradient}, to refine and verify it.
\emph{Research:} to refine hypotheses about the mechanisms behind our results, to provide feedback on the design of ablation experiments, and to interpret results.
\emph{Data and code:} to refine and extend the training and evaluation code and to build the ERA5 preprocessing pipeline; the core method and training code were implemented by the authors.
\emph{Writing:} to create and modify figures, draft and edit parts of the paper for readability, and identify relevant literature.

We did not use generative AI tools to generate the synthetic data set or to formulate the proposed method and its main claims; translation and qualitative data analysis are not applicable to this work.

We have reviewed all AI-assisted work: all derivations were worked through and checked by the authors, AI-assisted code was verified, and all references were checked against the original sources.


\bibliographystyle{abbrvnat}
\bibliography{references}

\clearpage
\appendix
\renewcommand{\thefigure}{A.\arabic{figure}}
\setcounter{figure}{0}
\renewcommand{\thetable}{A.\arabic{table}}
\setcounter{table}{0}

\renewcommand{\thefigure}{\thesection.\arabic{figure}}
\setcounter{figure}{0}
\section{Theory}
This section provides a definition of the energy score, a proof of \Cref{lemma:strict_properness} and a detailed derivation of the decision loss gradient in \cref{eq:ift_member}.
\subsection{Energy Score}
\label{app:energy_score}
The energy score \citep{gneiting2007strictly} evaluates a predictive distribution $\hat{F}$ against a realised outcome $y$ via
\begin{equation}
    \mathrm{ES}(\hat{F}, y) = \mathbb{E}_{Y \sim \hat{F}}\|Y - y\| - \frac{1}{2}\mathbb{E}_{Y, Y' \sim \hat{F}}\|Y - Y'\|
\end{equation}
For sample-based models, both expectations are approximated by Monte Carlo estimates over the same $M$ samples $\{\hat{y}_m\}_{m=1}^M$ forming the empirical distribution $\hat{F}_M$. We use the fair (unbiased) estimator~\citep{ferro2014fair},
\begin{equation}
    \mathrm{ES}_{\mathrm{fair}}(\hat{F}_M, y) = \frac{1}{M}\sum_{m=1}^M \|\hat{y}_m - y\| - \frac{1}{2M(M-1)}\sum_{m=1}^M\sum_{j=1}^M \|\hat{y}_m - \hat{y}_j\|.
    \label{eq:energy_score}
\end{equation}
For univariate targets the energy score reduces to the continuous ranked probability score (CRPS).

\subsection{Proof of \cref{lemma:strict_properness}}
\label{app:proof_strict_properness}

\begin{lemma}
\label{lemma:strict_properness}
Let $S_1$ be a strictly proper scoring rule and $S_2$ be a proper scoring rule.
Then $S = (1-w_d)\,S_1 + w_d\,S_2$ is strictly proper for any $w_d \in [0, 1)$.
\end{lemma}

\begin{proof}
The result follows immediately from the definitions. Let $F$ be the true distribution and $\hat{F} \neq F$ any forecast. Then
\begin{align*}
    \mathbb{E}_{Y \sim F}[S(\hat{F}, Y)]
    &= (1-w_d)\,\mathbb{E}[S_1(\hat{F}, Y)] + w_d\,\mathbb{E}[S_2(\hat{F}, Y)] \\
    &> (1-w_d)\,\mathbb{E}[S_1(F, Y)]      + w_d\,\mathbb{E}[S_2(F, Y)] \\
    &= \mathbb{E}_{Y \sim F}[S(F, Y)],
\end{align*}
where the strict inequality uses $\mathbb{E}[S_1(\hat{F}, Y)] > \mathbb{E}[S_1(F, Y)]$ (strict properness of $S_1$) with weight $(1-w_d) > 0$, and $\mathbb{E}[S_2(\hat{F}, Y)] \geq \mathbb{E}[S_2(F, Y)]$ (properness of $S_2$) with weight $w_d \geq 0$.
\end{proof}

\subsection{Derivation of per-sample gradient (\cref{eq:ift_member})}
\label{app:ift_derivation}

\paragraph{Problem setup.}
Define $f(a; \{\hat{y}_j\}_{j=1}^M) := \frac{1}{M}\sum_{j=1}^M c(a, \hat{y}_j)$; we write $f(a)$ for brevity, making the sample dependence explicit only where the argument requires it.
The optimal action solves
\begin{equation}
    \min_a\ f(a)
    \quad \text{s.t.} \quad
    a \geq a_\mathrm{min}, \quad a \leq a_\mathrm{max}.
    \label{eq:opt_problem}
\end{equation}
Since the feasible action set $[a_{\min}, a_{\max}]$ satisfies Slater's condition, the KKT conditions are necessary at the optimal action $a^*$.
Rewriting the constraints in standard form $g_1(a) = a_\mathrm{min} - a \leq 0$, $g_2(a) = a - a_\mathrm{max} \leq 0$, the Lagrangian is
\begin{equation}
    \mathcal{L}(a, \mu_1, \mu_2) = f(a) + \mu_1(a_\mathrm{min} - a) + \mu_2(a - a_\mathrm{max}),
\end{equation}
with KKT conditions at $a^*$:
\begin{align}
    &\text{(stationarity)}           & f'(a^*) - \mu_1 + \mu_2 &= 0, \label{eq:kkt_stat}\\
    &\text{(dual feasibility)}       & \mu_1,\, \mu_2 &\geq 0,\\
    &\text{(complementary slackness)}& \mu_1(a^* - a_\mathrm{min}) &= 0, \quad \mu_2(a_\mathrm{max} - a^*) = 0,
\end{align}
where $f'(a^*) := \partial f / \partial a\,(a^*)$.

\paragraph{Interior solution.}
When $a^*$ is interior, both constraints are inactive, so complementary slackness forces $\mu_1 = \mu_2 = 0$ and all KKT conditions reduce to the first-order condition:
\begin{equation}
    f'(a^*) = \frac{1}{M}\sum_{j=1}^M \frac{\partial c}{\partial a}(a^*, \hat{y}_j) = 0.
    \label{eq:foc}
\end{equation}
To obtain $\partial a^*/\partial \hat{y}_m$ we differentiate the optimality condition~\cref{eq:foc} implicitly with respect to $\hat{y}_m$.
\begin{quote}
\textit{Implicit function theorem~\citep{blondel2022efficient, krantz2002implicit} (IFT).}

If $F\colon \mathbb{R} \times \mathbb{R}^n \to \mathbb{R}$ is continuously differentiable with $F(x_0, \theta_0) = 0$ and $\partial F/\partial x \neq 0$ at $(x_0, \theta_0)$,
then there exists a unique smooth function $x^*(\theta)$ near $\theta_0$ satisfying $x^*(\theta_0) = x_0$ and $F(x^*(\theta), \theta) = 0$.
\end{quote}
Applying this with $x = a$, $\theta = \{\hat{y}_j\}_{j=1}^M$, and $F(a; \{\hat{y}_j\}_{j=1}^M) := f'(a; \{\hat{y}_j\}_{j=1}^M)$, the IFT requires $\partial F/\partial a \neq 0$ at $a^*$:
\begin{equation}
    \frac{\partial F}{\partial a}\bigg|_{a^*} = \frac{1}{M}\sum_{j=1}^M \frac{\partial^2 c}{\partial a^2}(a^*, \hat{y}_j) =: H > 0,
\end{equation}
where $H > 0$ follows from positive curvature of $f$ at $a^*$; the IFT therefore guarantees that $a^*$ is locally a smooth function of $\{\hat{y}_j\}_{j=1}^M$.
In our setting, $a^*$ is already defined as a function of $\{\hat{y}_j\}$; the role of the IFT is therefore to establish smoothness of this function, licensing the implicit differentiation below — not to establish the functional dependence itself.
Since $F(a^*(\{\hat{y}_j\}); \{\hat{y}_j\}_{j=1}^M) = 0$ holds for all $\{\hat{y}_j\}$, its total derivative with respect to any $\hat{y}_m$ is zero; applying the chain rule gives
\begin{equation}
    \left.\frac{\partial F}{\partial a}\right|_{a^*} \cdot \frac{\partial a^*}{\partial \hat{y}_m} + \left.\frac{\partial F}{\partial \hat{y}_m}\right|_{a^*} = 0,
    \qquad
    \left.\frac{\partial F}{\partial \hat{y}_m}\right|_{a^*} = \frac{1}{M}\frac{\partial^2 c}{\partial a\,\partial \hat{y}_m}(a^*, \hat{y}_m),
\end{equation}
where only the $j = m$ term in the sum contributes to $\partial F / \partial \hat{y}_m$.
Solving for $\partial a^*/\partial \hat{y}_m$ gives ~\cref{eq:ift_member} in \cref{sec:gradient}.

\paragraph{Boundary solution.}
Consider $a^* = a_\mathrm{max}$ (the case $a^* = a_\mathrm{min}$ is symmetric).
Since $g_1$ is inactive, complementary slackness gives $\mu_1 = 0$, and~\cref{eq:kkt_stat} becomes
\begin{equation}
    f'(a_\mathrm{max}; \{\hat{y}_j\}_{j=1}^M) = -\mu_2 \leq 0,
    \label{eq:boundary_grad}
\end{equation}
where the inequality follows from dual feasibility $\mu_2 \geq 0$.
Excluding the measure-zero case $f'(a_\mathrm{max}) = 0$ (the unconstrained minimiser coincides exactly with $a_\mathrm{max}$), we have $f'(a_\mathrm{max}) < 0$ strictly.

Since $f$ depends on $\hat{y}_m$ through the $m$-th summand, $f'(a_\mathrm{max}; \{\hat{y}_j\}_{j=1}^M)$ is continuous in $\hat{y}_m$ and therefore remains strictly negative under small perturbations $\hat{y}_m \to \hat{y}_m + \delta$.
Since $f'(a_\mathrm{max})$ remains strictly negative (the objective still pushes beyond the boundary), the constraint $g_2$ stays active and $a^*$ stays at $a_\mathrm{max}$; the multiplier $\mu_2$ adjusts to maintain~\cref{eq:kkt_stat} without $a^*$ moving.
Consequently,
\begin{equation}
    \frac{\partial a^*}{\partial \hat{y}_m} = 0.
\end{equation}




\subsection{Finite-sample approximation of the combined loss}
\label{app:finite_m}

\Cref{prop:psr} and the strict properness of \cref{eq:combined} are population-level results, stated for exact expectations. In training, both loss terms are evaluated on a finite number $M$ of model samples. This section shows that the resulting finite-$M$ objective remains approximately strictly proper.

\textbf{Setup.}
We drop the input $x$ for brevity. Let $F$ be the true conditional distribution, $\hat{F}$ the model's predictive distribution, and $\hat{F}_M = \frac{1}{M}\sum_{m=1}^M \delta_{\hat{y}_m}$ the empirical distribution of $M$ i.i.d.\ samples $\hat{y}_m \sim \hat{F}$, as drawn in one training step. Let $a^*(\hat{F})$ denote the Bayes action (\cref{eq:optimal_action}) and $a^*(\hat{F}_M)$ its sample-based counterpart computed by the optimisation layer (\cref{eq:optimal_action_mc}).
 Let
\begin{equation*}
    J(a) = \mathbb{E}_{Y \sim F}[c(a, Y)], \qquad a^*_F = a^*(F) = \arg\min_a J(a)
\end{equation*}
denote the true expected cost of action $a$ and the true optimal action. Taking expectations over the observation $Y \sim F$, the population objective of \cref{eq:combined} is
\begin{equation*}
    L_\infty(\hat{F}) = (1 - w_d)\, \mathbb{E}_{Y \sim F}\big[\mathrm{ES}(\hat{F}, Y)\big] + w_d\, J\big(a^*(\hat{F})\big),
\end{equation*}
since $\mathbb{E}_{Y \sim F}[S_c(\hat{F}, Y)] = J(a^*(\hat{F}))$. It is strictly proper, i.e.\ uniquely minimised at $\hat{F} = F$. Because the samples are redrawn at every step, training instead minimises the finite-$M$ objective
\begin{equation*}
    L_M(\hat{F}) = \mathbb{E}_{\hat{F}_M}\Big[(1 - w_d)\, \mathbb{E}_{Y \sim F}\big[\mathrm{ES}_{\mathrm{fair}}(\hat{F}_M, Y)\big] + w_d\, J\big(a^*(\hat{F}_M)\big)\Big],
\end{equation*}
where $\mathbb{E}_{\hat{F}_M}$ averages over sample draws and $\mathrm{ES}_{\mathrm{fair}}$ is the fair energy score estimator (\cref{app:energy_score}).

\textbf{Decomposition.}
The fair estimator is unbiased, $\mathbb{E}_{\hat{F}_M}[\mathrm{ES}_{\mathrm{fair}}(\hat{F}_M, y)] = \mathrm{ES}(\hat{F}, y)$ for all $\hat{F}$ and $y$~\citep{ferro2014fair}. The finite-$M$ effect therefore enters only through the decision term:
\begin{equation}
    L_M(\hat{F}) = L_\infty(\hat{F}) + w_d\, b_M(\hat{F}),
    \qquad
    b_M(\hat{F}) = \mathbb{E}_{\hat{F}_M}\big[J(a^*(\hat{F}_M))\big] - J\big(a^*(\hat{F})\big).
    \label{eq:finite_m_decomp}
\end{equation}
The bias $b_M$ is the expected difference in true cost between acting on $M$ samples and acting on the full predictive distribution.

\textbf{Size of the bias.}
Because it is computed from finitely many samples, the sample-based action $a^*(\hat{F}_M)$ scatters around $a^*(\hat{F})$ and systematically incurs excess cost (analogous to the overestimation bias of Q-learning~\citep{thrun1993issues,vanhasselt2010double}).
Write $\bar{a} = a^*(\hat{F})$ and $\Delta_M = a^*(\hat{F}_M) - \bar{a}$. Expanding the true expected cost around $\bar{a}$,
\begin{equation*}
    J\big(a^*(\hat{F}_M)\big) = J(\bar{a}) + J'(\bar{a})\, \Delta_M + \tfrac{1}{2}\, J''(\bar{a})\, \Delta_M^2 + o(\Delta_M^2),
\end{equation*}
and taking the expectation over sample draws gives
\begin{equation}
    b_M(\hat{F}) = J'(\bar{a})\, \mathbb{E}_{\hat{F}_M}[\Delta_M] + \tfrac{1}{2}\, J''(\bar{a})\, \mathbb{E}_{\hat{F}_M}[\Delta_M^2] + o(1/M).
    \label{eq:bias_b}
\end{equation}
Minimisers of empirical averages such as $a^*(\hat{F}_M)$ are called M-estimators; standard asymptotics give a scatter $\mathbb{E}[\Delta_M^2] = O(1/M)$~\citep[Ch.~5]{Vaart_1998} and a mean shift $\mathbb{E}[\Delta_M] = O(1/M)$~\citep{rilstone1996second}, so both terms are $O(1/M)$.
We assume that there exists a constant $C > 0$, independent of $M$ and $\hat{F}$, such that
\begin{equation}
    |b_M(\hat{F})| \leq C/M \quad \text{for all } M \text{ and all forecasts considered},
    \label{eq:finite_m_bias_bound}
\end{equation}
which requires well-determined optimal actions, i.e.\ unique minima with positive curvature.

\textbf{Steepness away from the truth.}
The finite-$M$ term $w_d\, b_M$ is of order $1/M$, but how far it can move the minimiser depends on how steeply the population loss $L_\infty$ rises away from $F$: if the rise is steep, a small perturbation barely shifts the minimum; if it is flat, even a small one can shift it far. Strict properness guarantees that $F$ is the unique minimiser of $L_\infty$, but not how steep this rise is. For the energy score, the rise is known exactly: with $X, X' \sim \hat{F}$ and $Y, Y' \sim F$ independent, the expected score exceeds that of the truth by half the energy distance~\citep{szekely2013energy},

\begin{equation*}
    \mathbb{E}_{Y \sim F}\big[\mathrm{ES}(\hat{F}, Y)\big] - \mathbb{E}_{Y \sim F}\big[\mathrm{ES}(F, Y)\big]
    = \tfrac{1}{2}\, \mathrm{ED}(\hat{F}, F),
    \qquad
    \mathrm{ED}(\hat{F}, F) = 2\,\mathbb{E}\|X - Y\| - \mathbb{E}\|X - X'\| - \mathbb{E}\|Y - Y'\|,
\end{equation*}
a squared distance between distributions that is zero if and only if $\hat{F} = F$. The decision term can only add to this, since $J(a^*(\hat{F})) \geq J(a^*_F)$, so
\begin{equation}
    L_\infty(\hat{F}) - L_\infty(F) \geq \tfrac{1 - w_d}{2}\, \mathrm{ED}(\hat{F}, F).
    \label{eq:finite_m_growth}
\end{equation}
The energy score thus guarantees that $L_\infty$ rises at least quadratically in the distance from $F$, at a rate set by $1 - w_d$ and independent of the cost function.

\textbf{Result.}
Two effects compete: (i) the growth of $L_\infty$ away from $F$ and (ii) the finite-$M$ term. A forecast $\hat{F}$ can only outperform $F$ under the training objective if the gain from (ii) exceeds the loss from (i). By \cref{eq:finite_m_decomp},
\begin{equation*}
    L_M(\hat{F}) - L_M(F) = \big[L_\infty(\hat{F}) - L_\infty(F)\big] + w_d\,\big[b_M(\hat{F}) - b_M(F)\big]
    \;\geq\; \underbrace{\tfrac{1 - w_d}{2}\, \mathrm{ED}(\hat{F}, F)}_{\text{(i) min.\ loss for leaving } F}
    \;-\; \underbrace{\frac{2\, w_d\, C}{M}}_{\text{(ii) max.\ finite-}M\text{ gain}},
\end{equation*}
using \cref{eq:finite_m_growth} and \cref{eq:finite_m_bias_bound} (each of $b_M(\hat{F})$ and $b_M(F)$ is at most $C/M$ in magnitude). Any forecast far enough from $F$ that (i) exceeds (ii) is thus beaten by the truth. If $F$ is among the forecasts considered, the minimiser $\hat{F}^*$ of $L_M$ satisfies $L_M(\hat{F}^*) \leq L_M(F)$ and must therefore lie where (i) does not exceed (ii), i.e.
\begin{equation}
    \mathrm{ED}(\hat{F}^*, F) \;\leq\; \frac{4\, w_d}{1 - w_d} \cdot \frac{C}{M}.
    \label{eq:finite_m_result}
\end{equation}
The minimiser of the finite-$M$ objective thus lies within $O(1/M)$ of the truth in energy distance.
The bound vanishes at $w_d = 0$, where $L_M$ reduces to the fair energy score and is exactly strictly proper for every $M$, and diverges as $w_d \to 1$, where the energy-score anchor disappears.

The result concerns the objective itself: computing the optimal action from finitely many samples moves its minimiser by at most $O(1/M)$ in energy distance. It does not guarantee that training reaches this target: neural networks parameterise only a restricted family of distributions, training data are finite, and stochastic optimisation may stop at a local optimum. The resulting errors are precisely what decision-aware training targets: where they affect the decisions, the decision loss directs the model's limited capacity towards reducing them.
The method thus trades a small, controlled distortion of the target for a correction of errors that can be considerably larger.

\textbf{Empirical check.}
The only approximate step above is the size of the bias $b_M$ in \cref{eq:bias_b}. We measure it on the synthetic task, where the true distribution is known, first at the truth and then for trained models.

\emph{At the truth}, the bias can be predicted: the first-order term vanishes since $a^*_F$ minimises $J$, and M-estimator asymptotics give
\begin{equation}
    b_M(F) \approx \frac{\mathrm{Var}_{Y \sim F}\big[\partial_a c(a^*_F, Y)\big]}{2\, J''(a^*_F)\, M},
    \label{eq:finite_m_bias}
\end{equation}
i.e.\ the bias is larger when the cost's sensitivity to the action varies strongly with the outcome, and smaller when the expected cost is sharply curved at its minimum. The measured bias converges to this prediction and hence decays as $1/M$ (\cref{fig:finite_m_bias_truth}): the median ratio falls from $2.2$ at $M = 5$ to $1.01$ at $M = 1000$. For small $M$, the prediction underestimates the bias, most strongly for inputs where the optimal action changes steeply with the proportion of the two modes.

\emph{For trained models} (\cref{fig:finite_m_bias_trained}), the true expected cost of acting on $M$ samples follows a straight line in $1/M$ for $M \geq 10$, so the bias decays as $1/M$ also away from the truth, as assumed in \cref{eq:finite_m_bias_bound}. At the training value $M = 10$, it raises the expected cost by only $2.4\%$ for $w_d = 0$ and $1.3\%$ for $w_d = 0.95$ (medians over $3 \times 3$ seeds).

\begin{figure}
    \centering
    \includegraphics[width=0.9\linewidth]{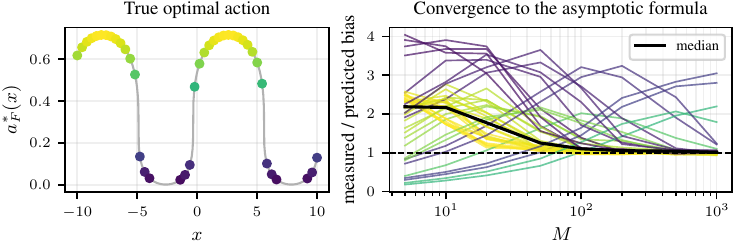}
    \caption{Finite-$M$ bias at the true distribution, per input (synthetic task). \emph{Left:} true optimal action $a^*_F(x)$; dots mark the inputs $x$ analysed on the right, at their optimal action (interior optima, $0.02 < a^*_F < 0.98$). \emph{Right:} ratio of measured bias (\cref{eq:finite_m_decomp}) to its asymptotic prediction (\cref{eq:finite_m_bias}), colours as in the left panel. The median (black) converges to $1$. Inputs on the plateaus of $a^*_F$ (yellow) converge quickly, as their optimal action is insensitive to the sampled proportion of the two modes; where $a^*_F$ changes steeply or lies near the action bound (green, purple), convergence is slower.}
    \label{fig:finite_m_bias_truth}
\end{figure}

\begin{figure}
    \centering
    \includegraphics[width=1.0\linewidth]{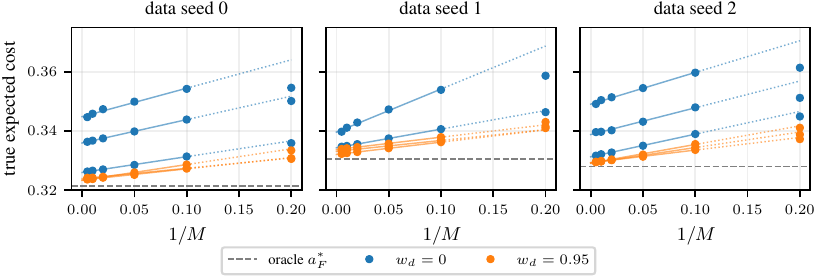}
    \caption{Finite-$M$ bias for trained models (synthetic task; one panel per data seed, one line per trained model, coloured by $w_d$). Markers: true expected cost of acting on $M$ model samples, $\mathbb{E}_{\hat{F}_M}[J(a^*(\hat{F}_M))]$, averaged over test inputs. Lines: linear fits in $1/M$ on $M \in \{10, \dots, 200\}$; the dotted continuation to $M = 5$ (not used in the fit) shows the deviation at small $M$. The intercept estimates the cost of acting on the full predictive distribution, $J(a^*(\hat{F}))$, so the vertical distance of a marker to the intercept is the finite-$M$ bias $b_M$ (\cref{eq:finite_m_decomp}). Dashed: cost of the true optimal action, $J(a^*_F)$; the gap to the intercept is the remaining model error. One $w_d = 0$ run of data seed 1, which did not recover the second mode, lies outside the plotted range.}
    \label{fig:finite_m_bias_trained}
\end{figure}

\textbf{Effect on training.}
Finally, we vary the ensemble size used during training, $M \in \{5, 10, 20, 50\}$, on the synthetic task. Decision-aware training improves all three metrics consistently across ensemble sizes (\cref{fig:synthetic_m_sensitivity}), so the choice of $M$ does not affect our conclusions. We use $M = 10$ throughout.

\begin{figure}[H]
    \centering
    \includegraphics[width=\linewidth]{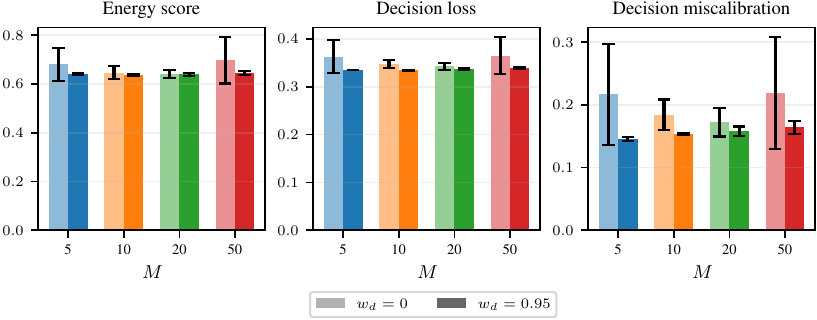}
    \caption{Sample size sensitivity of decision-aware training benefit (synthetic task).
    For each $M \in \{5, 10, 20, 50\}$ (number of samples drawn per input during training)
    we compare the pure energy-score baseline ($w_d=0$, transparent bars) against a decision-aware run ($w_d=0.9$, solid bars) on the synthetic task ($N_{train}=1600$, mean $\pm$ std over training seeds).
    Decision-aware training improves all three metrics consistently across $M$, showing robustness to the sample size.
    }
    \label{fig:synthetic_m_sensitivity}
\end{figure}


\section{Implementation Details}
\label{app:implementation}
Parts of our implementation build on code from \citet{kneissl2025improved}, released under the MIT licence.
Code for reproducing our experiments can be found at \url{https://anonymous.4open.science/r/decision-aware-training-2114}. 

\paragraph{Training.}
We use the Adam optimizer\citep{kingma2014adam} with learning rate $10^{-3}$ to train the models over $2000$ epochs. Loss scales of energy score and decision loss are estimated once before training and fixed throughout; optionally, a short energy score pre-training phase precedes scale estimation to ensure realistic samples and cost values. We perform 200 epochs of energy score pre-training for the wind power dispatch task, since we observe that untrained predictions yield outliers with non-realistic costs. We validate performance every 10 epochs on a validation set and apply early stopping with a patience of 600 epochs.

\paragraph{Distributional diffusion model.}
Distributional diffusion models~\citep{debortoli2025distributional, kneissl2025improved} keep the forward noising process and reverse chain of standard diffusion models~\citep{ho2020denoising} but change the denoiser. Instead of predicting a single estimate of the clean target $y_0$ from $(y_t, x)$, trained with a mean-squared error, the denoiser receives additional latent noise and outputs samples of $y_0$, trained with a strictly proper scoring rule such as the energy score so that they follow the full conditional distribution of $y_0$ given $(y_t, x)$. At inference, each reverse step uses a draw from this distribution instead of a point estimate.
The denoiser is an MLP with DDPM-style linear noise schedule over $T=100$ steps, with 5 hidden layers of dimension 50 and ReLU activations. $T=100$ is selected based on \cite{debortoli2025distributional, kneissl2025improved}, where improvements saturate for larger $T$. We also ran pilot experiments with larger $T$ but found no improvements. At each step the denoiser predicts a distribution over $\epsilon$ (equivalently $y_0$) given noisy $y_t$, drawing $M$ samples per step. Conditional input $x$ and a per-sample noise vector are injected via FiLM (Feature-wise Linear Modulation) conditioning at every hidden layer, giving each of the $M$ samples a unique modulation. 
\paragraph{Implicit generative model.}
The implicit generative model is an MLP with 5 hidden layers of dimension 50 and ReLU activations. It maps the concatenation of conditioning input $x$ and noise vector $\varepsilon \sim \mathcal{N}(0, I)$ directly to a sample $\hat{y}_m$; $M$ independent noise draws in parallel yield $M$ samples in a single forward pass.

\paragraph{Optimization layer.}
The optimal action $a^*$ is obtained by projected gradient descent on $\frac{1}{M}\sum_m c(a, \hat{y}_m)$ over $\mathcal{A}$, implemented within the optimization layer forward pass. We use 50 optimization steps during training, which small pilot experiments confirmed to be sufficient for convergence; at inference we use 200 steps for higher precision.
Second-order derivatives required for computing the decision loss gradient (\cref{eq:ift_member}) are computed via PyTorch \texttt{autograd}.

Decision-aware training adds computational overhead relative to standard energy score training through the optimisation layer, which is invoked at every forward pass.
The overhead scales with the number of samples $M$ (the optimisation objective is evaluated over all $M$ samples per batch element) and the number of iterations required to solve~\cref{eq:optimal_action} to convergence.

\paragraph{Computational resources}
We use an internal compute cluster with NVIDIA A100 GPUs. Training of all models required approximately 50 GPU hours.

\paragraph{Data}
The ERA5 dataset used in the real-world experiments s provided by the Copernicus Climate Change Service under a licence permitting research use~\citep{hersbach2020}.

\section{Additional Results}
This section presents additional experimental details and results for all three experiments in the main paper.

\subsection{Synthetic Decision Task}
\label{app:synthetic}

We use a 1600/200/400 train/val/test split across 3 data seeds and 3 training seeds, training a distributional diffusion model. The denoiser uses $M=10$ samples during training (\Cref{app:finite_m} analyses and ablates this choice). At inference, we draw $200$ samples from the learned predictive distributions via independent reverse chain runs.
The cost function belongs to the same sigmoid threshold family as the frost hazard cost; positive curvature of the expected cost at interior solutions is verified in \cref{app:frost_setup}.

\Cref{fig:synthetic_setup} shows the data generating process (bimodal conditional GMM with input-dependent mode weights and ground-truth optimal action $a^*(x)$) alongside the cost function.
\Cref{fig:synthetic_aggregate} shows aggregate energy score, decision loss, and decision calibration vs.\ $w_d$, pooled across all seeds.
\Cref{fig:synthetic_marginals} shows the predictive marginal distributions across seeds and $w_d$ values: mode 1 is overrepresented at $w_d=0$ (pure energy score), with mode 2 mass progressively recovering as $w_d$ increases.
\Cref{fig:synthetic_action_curves} shows the predicted $a^*(x)$ across data seeds and $w_d$: tracking improves consistently, though some of the troughs (boundary $a^*=0$, no gradient) remain a blind spot.
The optimal action is tracked correctly even at $w_d=1.0$ (no energy score anchor), but in this setting the underlying distribution degenerates to extreme values (\Cref{fig:synthetic_samples_scatter}). This provides empirical evidence that the energy score is necessary as a distributional anchor.
\Cref{fig:synthetic_mode_weight} shows the conditional mode 2 mass as a function of input $x$: dec loss partially recovers the true $w(x)$ around the peaks but overshoots at high $w_d$.
\begin{figure}[H]
    \centering
    \includegraphics[width=0.42\linewidth]{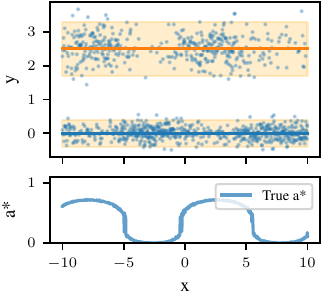}
    \hfill
    \includegraphics[width=0.4\linewidth]{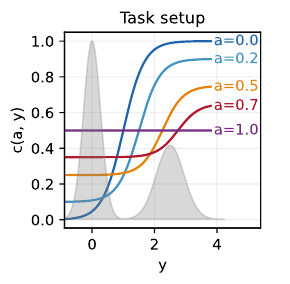}
    \caption{\emph{Left:} bimodal conditional GMM (top) and ground-truth optimal action $a^*(x)$ (bottom). \emph{Right:} cost function $c(a, y)$ for five protection levels $a$, overlaid with the marginal $p(y)$.}
    \label{fig:synthetic_setup}
\end{figure}

\begin{figure}[H]
    \centering
    \includegraphics[width=\linewidth]{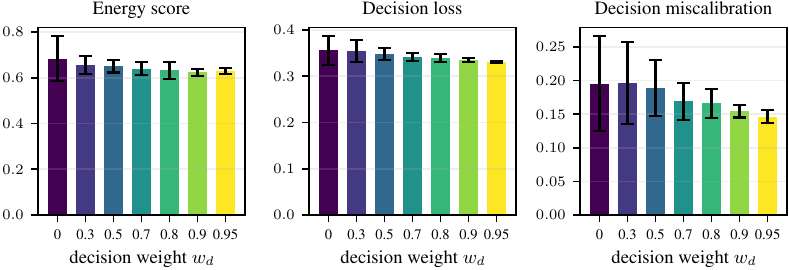}
    \caption{Aggregate energy score, decision loss, and decision miscalibration vs.\ $w_d$ (mean $\pm$ std, pooled across training and data seeds). energy score and decision loss improve slightly with $w_d$; average decision calibration improves more consistently throughout.}
    \label{fig:synthetic_aggregate}
\end{figure}

\begin{figure}[H]
    \centering
    \includegraphics[width=\linewidth]{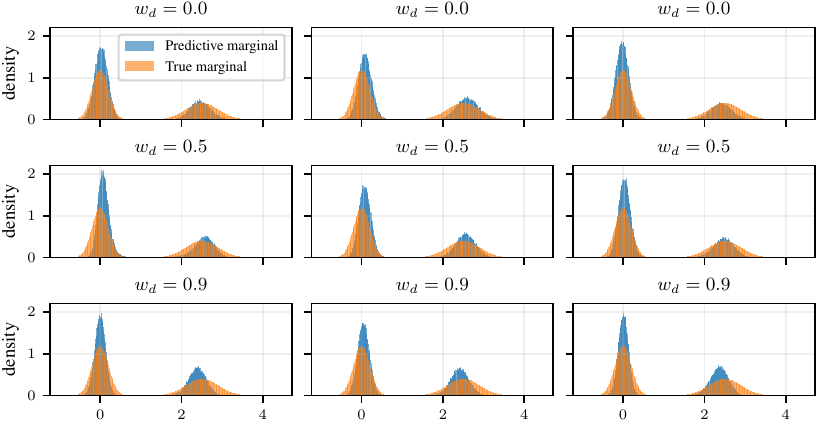}
    \caption{Predictive marginal distributions for $w_d \in \{0.0, 0.5, 0.9\}$ across three training seeds  (one seed per column). Mode 1 is overrepresented at $w_d=0$; mode 2 mass recovers with increasing $w_d$.}
    \label{fig:synthetic_marginals}
\end{figure}

\begin{figure}[H]
    \centering
    \includegraphics[width=\linewidth]{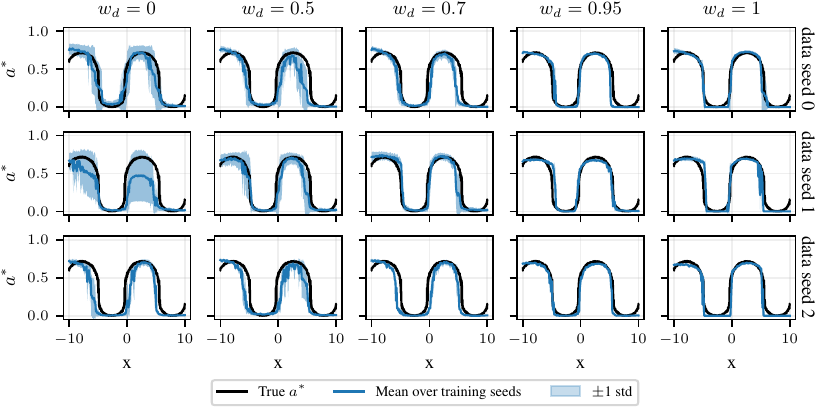}
    \caption{Predicted $a^*(x)$ (mean $\pm$ std over training seeds) vs ground truth across $w_d$ and data seeds (one seed per row). Pure energy score training ($w_d=0$) fails to track the ground truth action at the transitions between the extremes; decision-aware training  ($w_d > 0.0$) improves the tracking of $a*$ and gets closer to the ground truth for large $w_d$. Some of the troughs ($a^*=0$, boundary case) that exist for $w_d = 0$ remain a blind spot throughout. Notably, $a^*$ tracking remains accurate at $w_d = 1.0$ (pure decision loss training) despite the underlying predictive distribution degenerating (\Cref{fig:synthetic_samples_scatter}).}
    \label{fig:synthetic_action_curves}
\end{figure}

\begin{figure}[H]
    \centering
    \includegraphics[width=\linewidth]{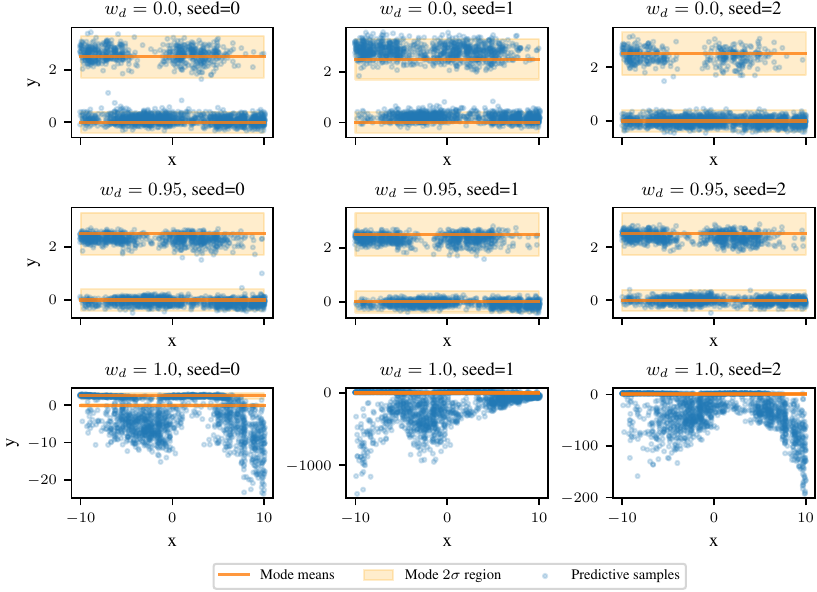}
    \caption{Predictive samples (blue) vs GMM mode means and $\pm 2\sigma$ regions (orange) across decision loss weights $w_d \in {0.0, 0.9, 1.0}$ and three data seeds. For $w_d \in {0.0, 0.9}$ samples cluster around both true modes. At $w_d = 1.0$ the energy score anchor is removed and the predictive distribution degenerates to extreme values (observe the different y axes). Combined with the preserved $a^*(x)$ tracking for $w_d=1.0$ (\Cref{fig:synthetic_action_curves}), this confirms empirically that the energy score is necessary as a distributional anchor.}
    \label{fig:synthetic_samples_scatter}
\end{figure}

\begin{figure}[H]
    \centering
    \includegraphics[width=0.9\linewidth]{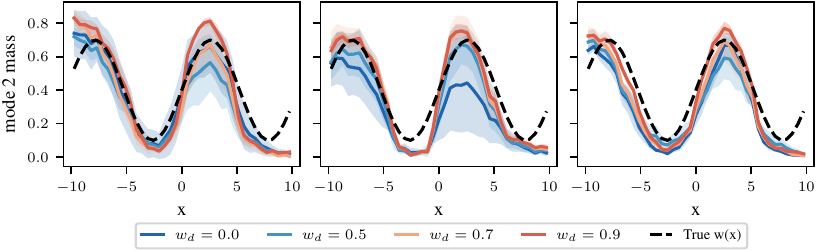}
    \caption{Conditional mode 2 mass estimated within 40 $x$-bins for different $w_d$, per data seed (one seed per panel). Lines show mean over training seeds; shaded band is $\pm$1 std. 
    Decision loss often improves the conditional mode mass estimate (moving towards $w(x)$) around the peaks and the transitions but can overshoot at high $w_d$.}
    \label{fig:synthetic_mode_weight}
\end{figure}
\subsection{Wind Power Dispatch}
In this section we report experimental details of the wind power dispatch experiment followed by additional results for the implicit generative model as well as results for the distributional diffusion model on this task.

\subsubsection{Experimental setup}
\label{app:windpower}

The power curve maps 10m wind speed to normalised power output via a differentiable sigmoid approximation~\citep{wang2019}; hub-height wind speeds are recovered via a power-law profile~\citep{abbes2012}.
The dispatch cost is $c(a,y) = -a + \lambda\,\mathrm{relu}(a - P(y))^2$, where $a \in [0,1]$ is the committed power fraction, $P(y)$ is the normalised power output, and $\lambda$ controls the shortfall penalty. We do not apply a penalty for power overproduction since this case can be avoided by windpower providers by dynamically reducing power output (curtailment) to match the promised amount \citep{bruninx2025}. In this case, only opportunity costs remain ($c=-a$), since a larger $a$ could have been chosen to increase revenue.

ERA5 inputs are observed wind speed at $t$, $t{-}6$\,h, and $t{-}12$\,h.
We use 2005--2019 for training, 2020 for validation, and 2021 for testing.
The test set contains 105 cut-off events ($v > 20$\,m/s at hub height), 230 rated, and 1074 ramp observations.
The implicit generative model uses $M=10$ samples during training and $M=500$ at inference. The distributional diffusion model uses $M=10$ samples for the probabilistic denoiser during training; at inference, 500 samples are drawn via independent reverse chain runs.

Positive curvature of $f(a) = \frac{1}{M}\sum_j c(a,\hat{y}_j)$ (\cref{app:ift_derivation}) at interior solutions holds analytically: 
since $\frac{\partial^2}{\partial a^2} c(a, \hat{y}_j) = 2\lambda$ when $P(\hat{y}_j) < a$ and $0$ otherwise, we have $H = (2\lambda/M)\,\#\{j : P(\hat{y}_j) < a^*\}$. 
At any interior optimum $a^*$, at least one sample must satisfy $P(\hat{y}_j) < a^*$, otherwise $f'(a^*) = -1 < 0$, contradicting optimality, so $H > 0$.

\begin{figure}[H]
    \centering
    \includegraphics[width=0.48\linewidth]{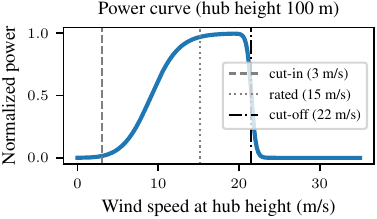}
    \hfill
    \includegraphics[width=0.48\linewidth]{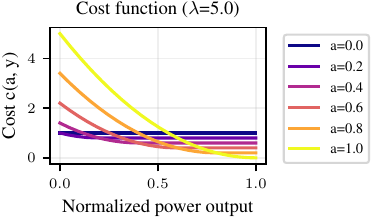}
    \caption{\emph{Left:} differentiable power curve used for converting wind speed at turbine hub height to normalized power output. \emph{Right:} dispatch cost $c(a,y)$ for $\lambda = 5$ as a function of committed power $a$ and observed power output $y$.}
    \label{fig:windpower_setup}
\end{figure}

\subsubsection{Additional Results: Distributional diffusion model}
\label{app:windpower_diffusion}

\Cref{fig:windpower_aggregate_diffusion} shows aggregate metrics vs $w_d$: energy score improves then overshoots at high $w_d$; aggregate decision loss is flat; decision calibration improves consistently for all $\lambda$.
\Cref{fig:windpower_regional_lam3_diffusion} shows the full regional breakdown for each $\lambda$.
\Cref{fig:windpower_cutoff_diffusion} shows the cut-off region ablation across $\lambda$.
\Cref{fig:windpower_rolling_q_diffusion} shows the rolling-window median shift.

\begin{figure}[H]
    \centering
    \includegraphics[width=\linewidth]{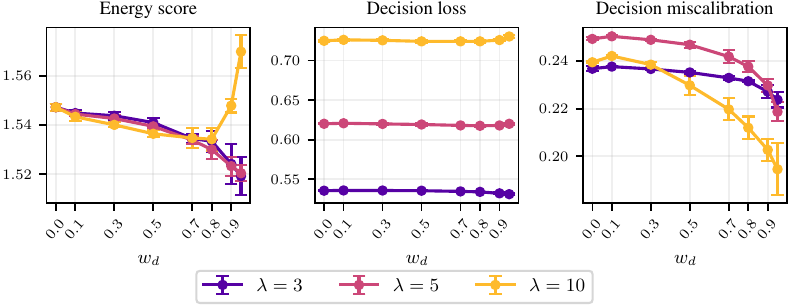}
    \caption{Aggregate metrics vs $w_d$ for $\lambda \in \{3,5,10\}$, distributional diffusion model. Energy score improves with $w_d$, but degrades at high $w_d$ for $\lambda=10$. Average decision loss stays relatively across $w_d$. Decision miscalibration improves with $w_d$, with stronger improvements at higher shortfall penalty $\lambda$.}
    \label{fig:windpower_aggregate_diffusion}
\end{figure}

\begin{figure}[H]
    \centering
    \includegraphics[width=0.8\linewidth]{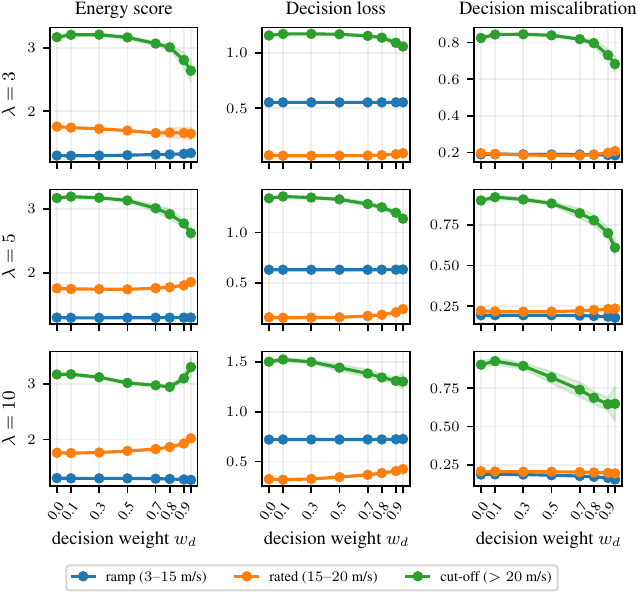}
    \caption{Regional metrics (cut-off, ramp, rated) vs $w_d$ for $\lambda=3$ (top), $\lambda=5$ (middle), and $\lambda=10$ (bottom), distributional diffusion model. Performance in the cut-off region improves with $w_d$, with stronger improvements at higher shortfall penalty $\lambda$; most rated region metrics degrade with $w_d$ across all $\lambda$.}
    \label{fig:windpower_regional_lam3_diffusion}
\end{figure}

\begin{figure}[H]
    \centering
    \includegraphics[width=\linewidth]{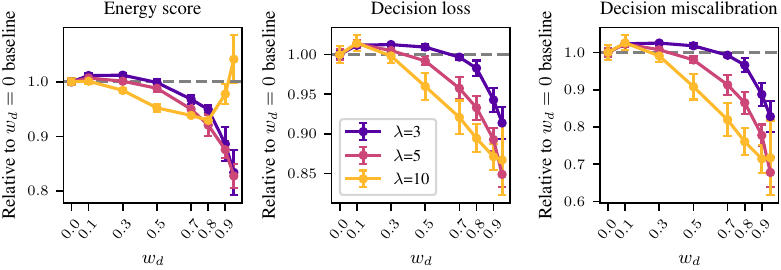}
    \caption{Cut-off region metrics vs $w_d$ across $\lambda \in \{3,5,10\}$, distributional diffusion model. 
    We observe a monotonic improvement pattern across $\lambda$ for $w_d \geq 0.5$. For larger $\lambda$, the decision loss gradient provides a stronger learning signal in the cut-off region leading to stronger improvements. Energy score degrades for large $\lambda$ and $w_d$, incidating overshoot (distributional fit degrades).}
    \label{fig:windpower_cutoff_diffusion}
\end{figure}

\begin{figure}[H]
    \centering
    \includegraphics[width=\linewidth]{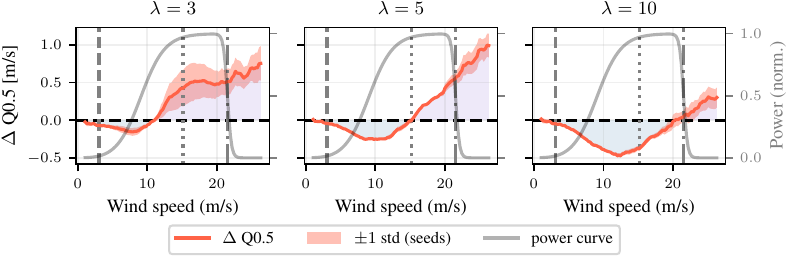}
    \caption{Rolling-window median shift $\Delta Q_{0.5}(v)$ conditioned on observed wind speed for $\lambda \in \{3,5,10\}$, distributional diffusion model. Similar pattern as for the implicit generative model: negative corrections in the ramp region and positive corrections above the cut-off threshold.}
    \label{fig:windpower_rolling_q_diffusion}
\end{figure}

\subsubsection{Implicit generative model results}
\label{app:windpower_implicit}

\Cref{fig:windpower_aggregate_implicit} shows the aggregate performance of the implicit generative model with respect to energy score, decision loss and decision miscalibration across seeds and decision weights $w_d$.
\Cref{fig:windpower_regional_implicit} shows the regional breakdown for $\lambda \in \{3,5,10\}$: in the cut-off region, decision loss and decision miscalibration improve consistently only for $\lambda = 10$.
With the biased uncorrected energy score estimator, which leaves larger tail errors, improvements appear for all $\lambda$ (\cref{app:biased_es}).
\Cref{fig:delta_cond_density} illustrates the effect of decision-aware training on the predictive distributions: the model shifts mass to more conservative values in the ramp region and sharpens forecasts near the cut-off threshold.
\Cref{fig:windpower_rolling_q_implicit} shows the rolling-window median shift $\Delta Q_{0.5}(v)$: the correction is negative in the ramp region and turns positive beyond cut-off.

\begin{figure}[H]
    \centering
    \includegraphics[width=\linewidth]{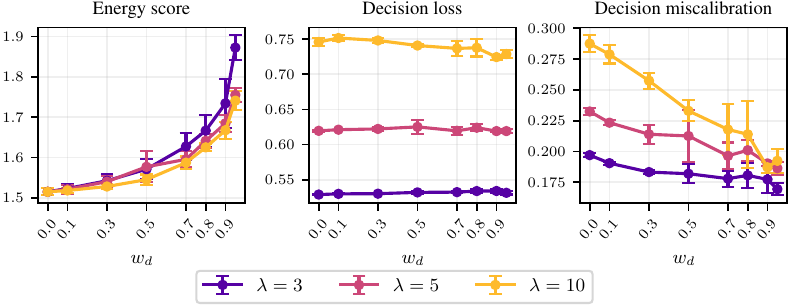}
    \caption{Aggregate metrics vs $w_d$ for shortfall penalty $\lambda \in \{3,5,10\}$, implicit generative model. Compared to the diffusion model (\Cref{fig:windpower_aggregate_diffusion}), we see a clearer trade-off between distributional accuracy (energy score) and decision quality across $w_d$. Aggregate decision loss remains relatively flat, whereas decision miscalibration improves with $w_d$, with stronger improvements at higher $\lambda$, mirroring the pattern observed for the diffusion model.}
    \label{fig:windpower_aggregate_implicit}
\end{figure}

\begin{figure}[H]
    \centering
    \includegraphics[width=0.8\linewidth]{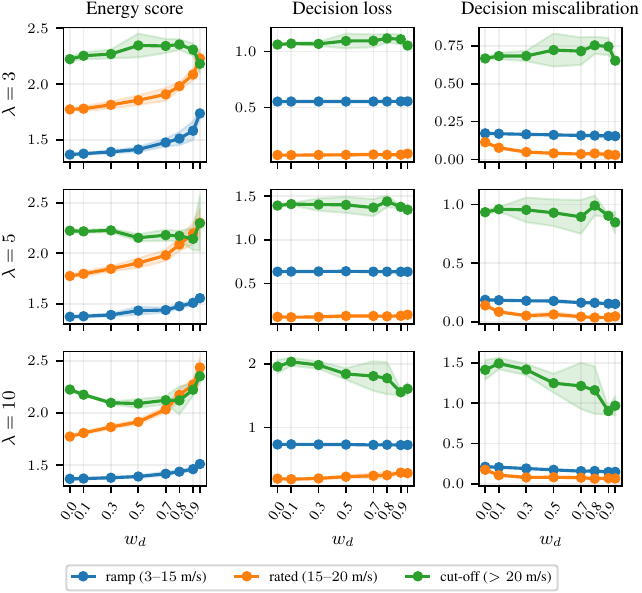}
    \caption{Regional metrics (cut-off, ramp, rated) vs $w_d$ for $\lambda=3$ (top row), $\lambda=5$ (middle row) and $\lambda=10$ (bottom row), implicit generative model. 
    For $\lambda \in \{3,5\}$ the baseline ($w_d=0$) is already well-calibrated in the cut-off region, leaving little room for improvement; for $\lambda=10$ a consistent improvement emerges through decision-aware training ($w_d>0$), as the larger shortfall penalty amplifies the cost of even small forecast errors in the cut-off region.}
    \label{fig:windpower_regional_implicit}
\end{figure}

\begin{figure}[H]
    \centering
    \includegraphics[width=\linewidth]{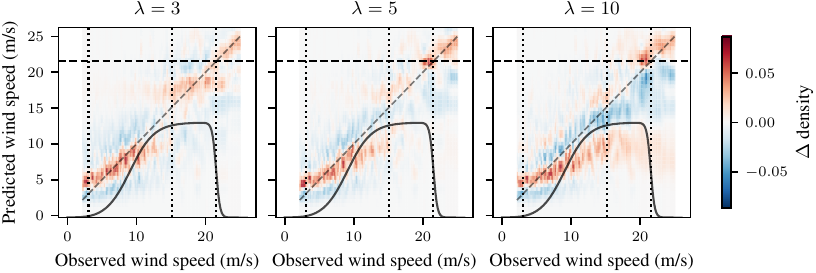}
    \caption{Difference in conditional forecast density (conditioned on observed wind speed) between $w_d=0.7$ and $w_d=0.0$ (red: density gain, blue: density loss), for $\lambda \in \{3,5,10\}$ (left to right). The gray curve shows the power curve (observed wind speed axis). In the ramp region, decision loss shifts predicted mass below the diagonal, reflecting more conservative dispatch. Near cut-off, the model moves mass into the tail to anticipate cut-off costs, with the correction growing with $\lambda$: at $\lambda=5$ the gained mass aligns with the diagonal, indicating that the model produces more confident forecasts near cut-off that are also well-aligned with observed wind speeds. At $\lambda=10$ the gained mass around cut-off shifts above the diagonal, indicating that the model exaggerates predicted wind speeds to hedge against the larger shortfall penalty.}
    \label{fig:delta_cond_density}
\end{figure}

\begin{figure}[H]
    \centering
    \includegraphics[width=\linewidth]{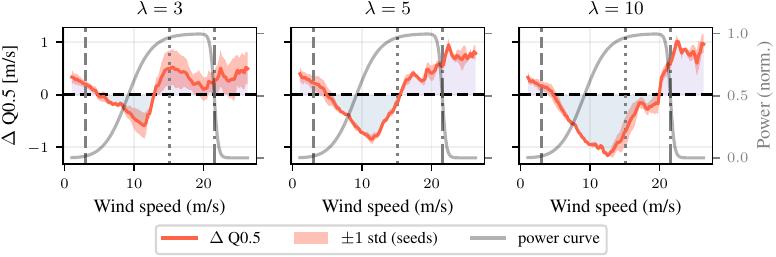}
    \caption{Rolling-window median shift $\Delta Q_{0.5}(v)$ conditioned on observed wind speed for $\lambda \in \{3,5,10\}$, implicit generative model. 
    Correction is negative in the ramp (more conservative dispatch) and positive in the cut-off region (more mass in costly tail where power drops to zero). This effect strengthens with $\lambda$. The model learns different correction behavior across outputs that reflects the cost function structure. Window size $\pm 1.5 \text{ m/s}$.}
    \label{fig:windpower_rolling_q_implicit}
\end{figure}

\subsubsection{Decision-aware training with the uncorrected energy score}
\label{app:biased_es}

The main experiments use the unbiased fair energy score estimator (\cref{app:energy_score}). A common alternative, and the default in widely used scoring-rule libraries, is the uncorrected estimator, which normalises the spread term by $M^2$ rather than $M(M-1)$:
\begin{equation*}
    \mathrm{ES}_{\mathrm{uncorr}}(\hat{F}_M, y) = \frac{1}{M}\sum_{m=1}^M \|\hat{y}_m - y\| - \frac{1}{2M^2}\sum_{m=1}^M\sum_{j=1}^M \|\hat{y}_m - \hat{y}_j\|.
    \label{eq:energy_score}
\end{equation*}
Since the $M$ diagonal terms vanish, the spread term is underestimated by a factor $(1 - 1/M)$ relative to the fair estimator. Spread is thus under-rewarded relative to proximity to the observation, and minimisation favours underdispersed forecasts~\citep{ferro2014fair}, with the tails most affected.
This provides a controlled setting to increase the residual error that energy score training leaves in the tail, and thus to test whether the benefit of decision-aware training grows with it.
We repeat the wind power dispatch experiment with the implicit generative model using the uncorrected estimator, keeping all other settings fixed ($\lambda \in \{3, 5, 10\}$, same $w_d$ grid and seeds).

\textbf{Aggregate results.}
The uncorrected baseline ($w_d = 0$) starts with higher decision loss and decision miscalibration than its fair counterpart (\cref{fig:windpower_aggregate_implicit_biased}; cf.\ \cref{fig:windpower_aggregate_implicit}). Increasing $w_d$ reduces both, reaching approximately the level obtained under the fair estimator; the effect grows with $\lambda$. The energy score remains flat up to moderate $w_d$ and increases at large $w_d$.

\textbf{Regional results.}
The improvement concentrates in the cut-off region (\cref{fig:windpower_regional_implicit_biased}). Under the uncorrected estimator, decision loss and decision miscalibration in this region decrease with $w_d$ for all $\lambda$, with larger $\lambda$ yielding stronger improvements; under the fair estimator, a comparable improvement appears only for the largest penalty, $\lambda = 10$ (\cref{fig:windpower_regional_implicit}). The energy score in the cut-off region also improves under the uncorrected estimator, approaching the fair baseline: decision-aware training corrects the underdispersed tail itself, not only the resulting decisions. The rated region degrades as in the fair case, since $\partial P / \partial y \approx 0$ and $a^* \approx 1$ leave the energy score as the only active signal; the ramp region is largely unaffected.

\textbf{Interpretation.}
The results confirm that the benefit of decision-aware training grows with the error the scoring rule leaves in cost-relevant regions (\cref{app:scale}): with the larger tail error of the uncorrected estimator, improvements appear for all $\lambda$, and at large $w_d$ the two estimators attain nearly identical decision quality. The correction is local, however: in the ramp and rated regions the uncorrected estimator remains worse. The fair estimator and the decision loss are therefore complementary: the former reduces distributional error everywhere, the latter corrects the residual error that matters for the decision. The conditional density shift induced by the decision loss (not shown) is qualitatively identical to the fair case (\cref{fig:delta_cond_density}), with approximately twice the magnitude. This is consistent with two sources of tail underdispersion acting in the same direction: the estimator bias, and the density-proportional allocation of the energy score gradient. The fair estimator removes only the former, leaving a smaller residual that decision-aware training corrects when the penalty is large enough.
We report this comparison for the implicit generative model only; for the distributional diffusion model, the two estimators yield similar results.

\begin{figure}[H]
    \centering
    \includegraphics[width=\linewidth]{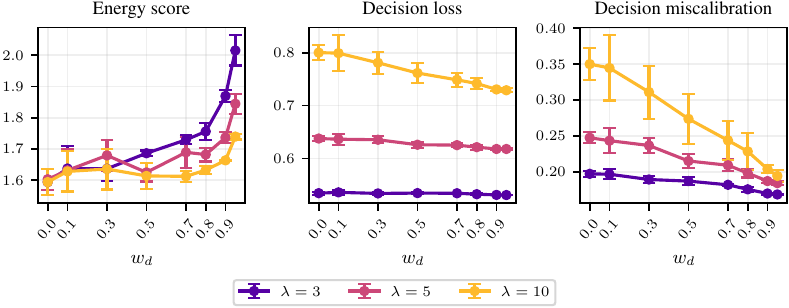}
    \caption{Wind power dispatch, implicit generative model, uncorrected energy score estimator: aggregate metrics vs.\ $w_d$ for $\lambda \in \{3, 5, 10\}$ (mean $\pm$ 1 std over seeds). Compare \cref{fig:windpower_aggregate_implicit} (fair estimator).}
    \label{fig:windpower_aggregate_implicit_biased}
\end{figure}

\begin{figure}[H]
    \centering
    \includegraphics[width=0.8\linewidth]{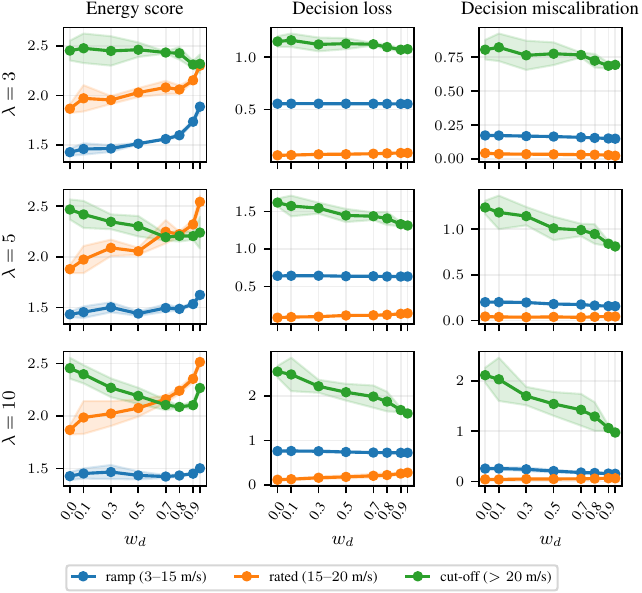}
    \caption{As \cref{fig:windpower_aggregate_implicit_biased}, conditional on the power curve region of the observed wind speed. Compare \cref{fig:windpower_regional_implicit} (fair estimator).}
    \label{fig:windpower_regional_implicit_biased}
\end{figure}

\subsection{Frost Protection}
In this section we report experimental details of the frost protection experiment followed by additional results for the implicit generative model as well as results for the distributional diffusion model on this task.
\subsubsection{Experimental setup}
\label{app:frost_setup}
ERA5 inputs are 2\,m temperature from a $3{\times}3$ spatial patch (centered around the target grid point) at three time steps (27 features), restricted to November--March.
The implicit generative model uses $M=10$ samples during training and $M=500$ at inference. The distributional diffusion model uses $M=10$ samples from the probabilistic denoiser during training; at inference, 500 samples are drawn via independent reverse chain runs.
The cost function is a sigmoid threshold at $0$\,\textdegree C; $\alpha = \frac{FP}{FN}$ controls the asymmetry between false positive and false negative costs (\cref{fig:frost_cost_function}).
The gradient derivation (\cref{app:ift_derivation}) requires strict convexity of $f(a) = \frac{1}{M}\sum_j c(a, \hat{y}_j)$ at interior solutions $a^*$. 
This was verified empirically across 100,000 randomly drawn sample sets of size $M=50$ from the marginal training distribution, with $H > 0$ (strict convexity) in all cases.

\begin{figure}[H]
    \centering
    \includegraphics[width=0.5\linewidth]{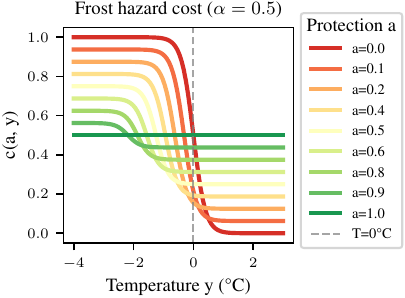}
    \caption{Frost protection cost function for $\alpha=0.5$.}
    \label{fig:frost_cost_function}
\end{figure}

\subsubsection{Additional results: implicit generative model}
\label{app:frost}

The main text shows aggregate metrics for all $\alpha$ values and the conditional KDE for $\alpha=0.3$.
\Cref{fig:frost_kde} shows the conditional KDE for all three $\alpha$ values at $w_d=0.3$, alongside the shared $w_d=0$ baseline.
\Cref{fig:frost_marginal} shows the marginal 2m temperature (T2m) distributions for $\alpha=0.5$: the marginal gets fitted better around the decision-critical freezing point.
\Cref{fig:frost_conditional} shows the conditional breakdown by frost and no-frost regime for $\alpha \in \{0.2,0.3,0.5\}$: the no-frost regime improves in terms of decision loss and decision calibration while the frost regime degrades slightly or stays constant across $w_d$. \Cref{fig:frost_rolling_dq} shows the quantile shift in the predictive distributions introduced through decision-aware training: quantiles closer to the cost-sensitive threshold region get shifted more strongly, supporting the claim of probability mass being redistributed around the decision threshold.

\begin{figure}[H]
    \centering
    \includegraphics[width=\linewidth]{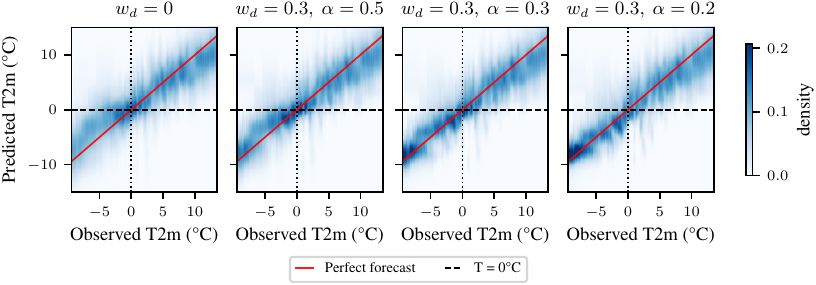}
    \caption{Rolling-window conditional predictive density of temperature forecasts for $w_d=0$ (leftmost, shared baseline) and $w_d=0.3$ for $\alpha=0.2$, $\alpha=0.3$, and $\alpha=0.5$ (left to right), implicit generative model. 
    A warm bias for sub-zero temperatures visible at $w_d=0$ is corrected at $w_d=0.3$ progressively with $\alpha$, with the predictive distribution concentrating closer to the diagonal in the frost region. Window size $\pm 1^\circ$C.}
    \label{fig:frost_kde}
\end{figure}

\begin{figure}[H]
    \centering
    \includegraphics[width=\linewidth]{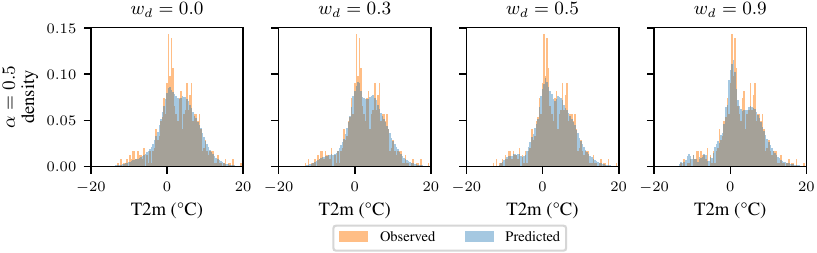}
    \caption{Marginal T2m distribution (observed vs predicted) for $\alpha=0.5$ at $w_d \in \{0.0, 0.3, 0.5, 0.9\}$, implicit generative model. As $w_d$ grows, the model fits the marginal distribution better around the (decision-critical) freezing point.}
    \label{fig:frost_marginal}
\end{figure}

\begin{figure}[H]
    \centering
    \includegraphics[width=0.8\linewidth]{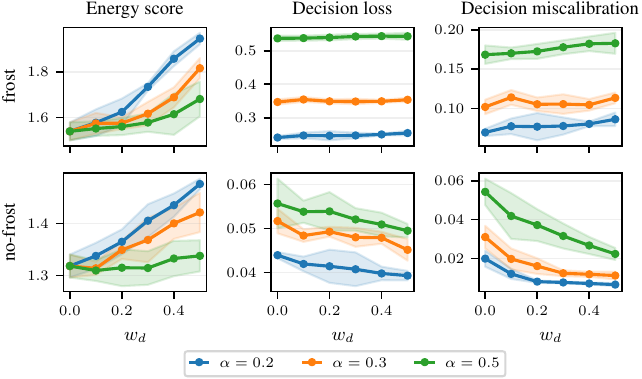}
    \caption{Conditional metrics vs $w_d$ for $\alpha \in \{0.2,0.3,0.5\}$ conditioned on observed temperatures (below or above $0$ \textdegree C), implicit generative model: frost region (top row) and no-frost region (bottom row). At $\alpha = 0.5$, the no-frost region improves while the frost region degrades slightly (decision loss and miscalibration). Lowering $\alpha$ (increasing cost for missed frost) degrades energy score more quickly.}
    \label{fig:frost_conditional}
\end{figure}

\begin{figure}[H]
    \centering
    \includegraphics[width=0.9\linewidth]{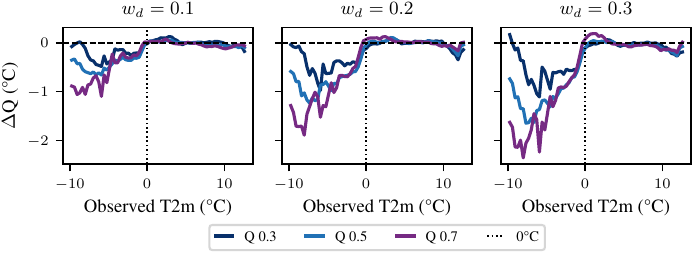}
    \caption{Rolling quantile shift (relative to $w_d=0$) for $w_d \in \{0.1, 0.2, 0.3\}$, implicit generative model ($\alpha=0.3$, window $\pm 1$\textdegree C).
    For observations near the threshold (obs T2m $\approx 0$\textdegree C), all quantiles shift more uniformly as the whole ensemble straddles the cost-active region.
    For deep-cold observations (obs T2m $< -5$\textdegree C), the upper quantiles (Q $0.7$) are closest to the 0\textdegree C threshold and therefore sit closest to the cost-active region, shifting substantially more than the lower quantiles (Q $0.3$), which are further from the threshold and receive a weaker gradient.
    This quantile ordering directly illustrates the sample-level redistribution effect of the decision loss gradient predicted by the gradient analysis in \cref{sec:gradient}.
    The effect scales consistently with $w_d$ across all three panels.}
    \label{fig:frost_rolling_dq}
\end{figure}

\subsubsection{Distributional diffusion model results}
\label{app:frost_diffusion}

For the distributional diffusion model, the decision loss does not improve the forecast distributions as expected.
The baseline ($w_d=0$) already exhibits a cold bias: for sub-zero observations, the conditional forecast distribution is shifted toward colder temperatures relative to the observed values, as visible in the KDE plots (\cref{fig:frost_diffusion_kde}).
Decision-aware training ($w_d > 0)$ does not correct this artifact.
energy score degrades monotonically with $w_d$ in all $\alpha$ settings, yet decision miscalibration improves for $\alpha < 0.5$.
Notably, the same sigmoid cost function with a distributional diffusion model yields clear improvements on the synthetic decision task (\cref{sec:experiments}), so the failure is not inherent to the cost structure or model class.
We do not have a complete explanation for this failure mode, one hypothesis is that samples from the baseline model ($w_d=0$) are already outside of the strong gradient region of the cost function (too cold), thus the decision loss has no handle to correct them.
We also note that the reverse chain dynamics at inference time are not captured by the per-step training objective (\cref{sec:conclusion}), potentially introducing compounding errors at inference.

\begin{figure}[H]
    \centering
    \includegraphics[width=\linewidth]{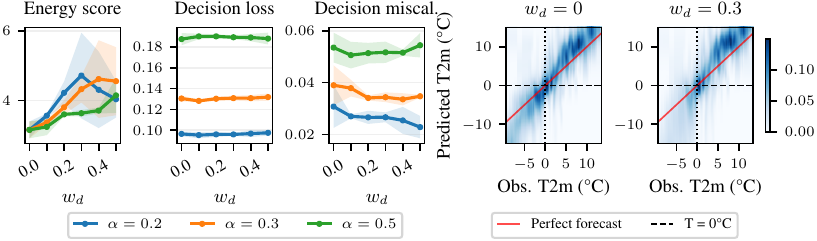}
    \caption{Frost protection results, distributional diffusion model. Same plot as for implicit generative model in \Cref{fig:frost_main}. \emph{Left (three panels):} Aggregate energy score, decision loss, and decision miscalibration vs $w_d$ for $\alpha \in \{0.2,0.3,0.5\}$ (mean $\pm$ 1 std over seeds). \emph{Right (two panels):} Conditional predictive density for $w_d=0$ and $w_d=0.3$ ($\alpha=0.3$) conditioned on observed temperature. The baseline ($w_d=0$) already exhibits a cold bias for sub-zero observations visible in the conditional KDE. Decision-aware training ($w_d = 0.3$) does not correct this bias. Energy score degrades monotonically with larger $w_d$ for all $\alpha$ while decision loss stays flat throughout. For $\alpha<0.5$, decision calibration nonetheless improves with $w_d$, suggesting the expected cost of the optimal action becomes better aligned with the observed cost even as the overall distributional quality degrades.}
    \label{fig:frost_diffusion_kde}
\end{figure}

\section{Applicability and Scale}
\label{app:scale}
Decision-aware training is beneficial whenever energy score training yields suboptimal approximation errors in cost-relevant regions. This regime arises broadly in practice: approximation errors are unavoidable for any model trained with finite data and limited capacity, and energy score training provides no guarantee that these errors avoid cost-relevant regions.

To characterise how the benefit of decision-aware training scales with training data size, we vary the training set size $N_\mathrm{train}$ on the synthetic task while holding the model architecture and task fixed. 
The main-text experiments use $N_\mathrm{train}=1600$; 
here we sweep $N_\mathrm{train} \in \{400, 800, 1600, 3200, 6400, 12800\}$ with three training seeds per size.

Figures~\ref{fig:ntrain_metrics} and~\ref{fig:ntrain_a_star} summarise the result. 
The relative improvements in decision loss and decision miscalibration shrink as $N_\mathrm{train}$ grows and vanish at $N_\mathrm{train}=12800$. The improvement in optimal action ($a^*$) MAE (mean absolute error of $a^*(x)$ against the ground-truth $a^*(x)$ across test inputs) also shrinks but remains visible across the full range studied. 
The $a^*(x)$ recovery panels (\Cref{fig:ntrain_a_star}) show the same story qualitatively: the baseline ($w_d=0$) fails to track the true $a^*(x)$ at small $N_\mathrm{train}$ and only converges to it for the largest sizes, whereas decision-aware training tracks $a^*(x)$ accurately throughout.

Operational forecasting tasks typically fall into the regime where decision-aware training helps: meteorological time series exhibit strong temporal and spatial autocorrelation that reduces the effective sample size below the nominal dataset size. Additionally, rare high-impact events are often the most costly, but by definition less frequent, and thus tend to be underweighted by energy score training.

\begin{figure}[H]
    \centering
    \includegraphics[width=1.0\linewidth]{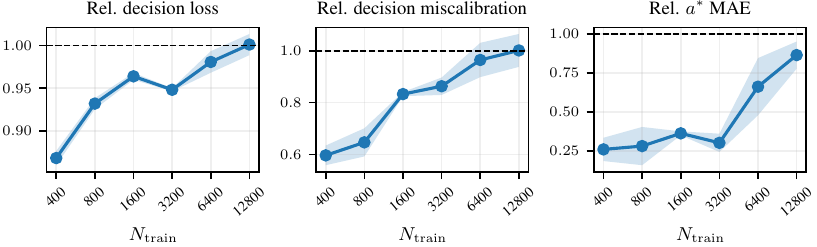}
    \caption{Relative performance of decision-aware training ($w_d=0.95$) against the energy-score baseline ($w_d=0$) on the synthetic decision task (single data seed) for different training set sizes $N_\mathrm{train}~\in~\{400, 800, 1600, 3200, 6400, 12800\}$. Values below $1.0$ (dashed line) indicate improvement over baseline. Decision loss and decision miscalibration improvements shrink as $N_\mathrm{train}$ grows and vanish at $N_\mathrm{train}=12800$. Optimal action $a^*$ MAE improvement also shrinks for larger $N_\mathrm{train}$ but persists across the range studied. Shaded bands show $\pm 1$ std over training seeds.}
    \label{fig:ntrain_metrics}
\end{figure}

\begin{figure}[H]
    \centering
    \includegraphics[width=1.0\linewidth]{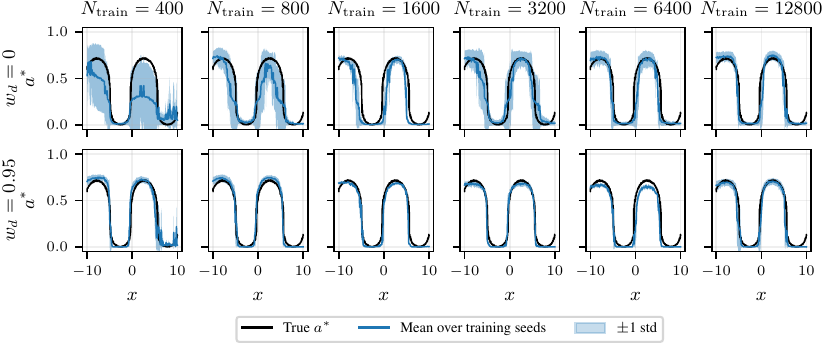}
    \caption{Optimal action $a^*(x)$ recovery on the synthetic decision task (single data seed) across training set sizes $N_\mathrm{train} \in {400, 800, 1600, 3200, 6400, 12800}$ for the energy-score baseline ($w_d=0$, top row) and decision-aware model ($w_d=0.95$, bottom row). Blue lines show the mean over training seeds, shaded bands show $\pm 1$ std. Decision-aware training recovers the true $a^*$ (black) more accurately. Improvements are largest at small-to-medium $N_\mathrm{train}$ where the baseline fails to track $a^*$; at larger $N_\mathrm{train}$ the baseline already tracks $a^*$ well, leaving less room for improvement. Nonetheless, $a^*$ recovery gains remain visible across the full range studied.}
    \label{fig:ntrain_a_star}
\end{figure}

\section{Choosing the decision weight $w_d$ in practice}
\label{app:wd_guidance}

The decision weight $w_d$ is a user-preference knob rather than a hyperparameter with a single optimum. It encodes a trade-off between distributional accuracy (small $w_d$) and decision specialization (large $w_d$), and the appropriate value depends on how the user weights these two goals and not on the task alone.

Nevertheless, the loss construction makes $w_d$ interpretable and comparable across tasks. The loss-scale estimation step (\Cref{sec:architectures_training}) normalizes the energy score and the decision loss to comparable magnitude before training. This removes the cost-magnitude degree of freedom, so that $w_d \in [0, 1]$ carries a consistent meaning: the fraction of training signal allocated to the decision objective. Without this normalization, the useful range of $w_d$ would depend on the cost scale: close to zero (e.g.\ $[0, 0.1]$) for large-magnitude costs and squeezed toward one (e.g.\ $[0.9, 0.99]$) for small ones; no shared grid or cross-task comparison would be possible.

In practice, a coarse grid over $w_d$ (e.g.\ $\{0.3, 0.5, 0.7\}$) is sufficient to identify the optimal range. Each candidate is evaluated on both validation energy score and validation decision cost (or decision miscalibration), and the user selects the largest decision improvement that remains within an acceptable degradation of the energy score. Task structure informs the starting point: when high-cost events are rare (e.g.\ wind-power cut-off), a larger $w_d$ is needed for the decision term to contribute sufficient gradient signal, whereas for frequent events (e.g.\ frost) a smaller $w_d$ suffices.

\end{document}